\pdfoutput=1
\documentclass[journal]{IEEEtran}

\ifCLASSINFOpdf
\else
\fi

\usepackage[numbers,sort&compress]{natbib}

\usepackage[utf8]{inputenc} 
\usepackage[T1]{fontenc}    
\usepackage{hyperref}       
\usepackage{url}            
\usepackage{booktabs}       
\usepackage{amsfonts}       
\usepackage{nicefrac}       
\usepackage{microtype}      
\usepackage{multirow}

\usepackage{subfigure}

\usepackage{ifthen}
\usepackage{bbm}
\usepackage{amssymb}
\usepackage{algorithmicx}
\usepackage[cmex10]{amsmath}
\usepackage{amsthm}
\usepackage{mathtools}
\usepackage{bm}
\usepackage[inline, shortlabels]{enumitem}
\usepackage[electronic]{ifsym}
\usepackage{caption}
\usepackage{physics}
\usepackage{tabularx}

\usepackage{tikz}
\usetikzlibrary{fit, shapes.geometric, shapes.misc, arrows.meta, positioning, decorations.markings}
\tikzset{>={Latex[width=1.2mm,length=1.5mm]}}
\usepackage{smartdiagram}
\usesmartdiagramlibrary{additions}

\usepackage{colortbl}
\definecolor{LightCyan}{rgb}{0.7,1,1}
\definecolor{LightGrey}{rgb}{0.95,0.95,0.95}
\definecolor{DarkGrey}{rgb}{0.4,0.4,0.4}
\definecolor{DarkBlue}{rgb}{0,0,0.4}

\DeclareMathOperator*{\supp}{supp}
\DeclareMathOperator*{\sign}{sgn}
\DeclareMathOperator*{\vari}{var}
\DeclareMathOperator*{\cov}{cov}
\DeclareMathOperator*{\E}{\mathbb{E}}
\DeclareMathOperator*{\prob}{Pr}
\DeclareMathOperator*{\proj}{\mathcal{P}_K}
\DeclareMathOperator*{\support}{\mathcal{S}_K}
\DeclareMathOperator*{\sparse}{\mathcal{H}_K}
\DeclareMathOperator{\bigo}{{\mathcal{O}}}
\DeclareMathOperator{\littleo}{{\scriptstyle\mathcal{O}}}

\DeclareMathOperator*{\polylog}{polylog}
\let\originalleft\left
\let\originalright\right
\renewcommand{\left}{\mathopen{}\mathclose\bgroup\originalleft}
\renewcommand{\right}{\aftergroup\egroup\originalright}

\newcommand{\bw}{{\bm w}}
\newcommand{\bx}{{\bm x}}
\newcommand{\be}{{\bm e}}
\newcommand{\by}{{\bm y}}
\newcommand{\ba}{{\bm a}}
\newcommand{\bs}{{\bm s}}

\newcommand{\bl}{{\bm l}}
\newcommand{\bp}{{\bm p}}

\newcommand{\bpsi}{{\bm \psi}}
\newcommand{\bweq}{{\bw_{\mathrm{eq}}}}
\newcommand{\bweqi}{{\bw_{\mathrm{eq}}^{\left\{i\right\}}}}
\newcommand{\bweqt}{{\bw_{\mathrm{eq}}^{\left\{t\right\}}}}

\newcommand{\bweqk}{{\bw_{\mathrm{eq}}^{\left\{k\right\}}}}

\newcommand{\beq}{b_{\mathrm{eq}}}
\newcommand{\beqi}{{b_{\mathrm{eq}}^{\left\{i\right\}}}}

\newcommand{\overbar}[1]{\mkern 1.5mu\overline{\mkern-1.5mu#1\mkern-1.5mu}\mkern 1.5mu}

\newcommand{\sbar}{{\overbar{s}}}

\newcommand{\bxbar}{\bm{{\overbar{x}}}}

\newcommand{\bsbar}{{\bm{\overbar{s}}}}
\newcommand{\bweqbar}{{\overbar{\bw}_{\mathrm{eq}}}}
\newcommand{\beqbar}{{\overbar{b}_{\mathrm{eq}}}}

\newcommand{\bxhat}{\bm{\hat{x}}}

\newcommand{\subM}{{\mathchoice{}{}{\scriptscriptstyle}{}M}}

\theoremstyle{plain}
\newtheorem{theorem}{Theorem}
\newtheorem{lemma}{Lemma}
\newtheorem{proposition}{Proposition}

\theoremstyle{definition}

\theoremstyle{remark}
\newtheorem*{remark}{Remark}

\newtheorem*{assumption}{Assumption}
\newtheorem*{assumptions}{Assumptions}

\newenvironment{proof-sketch}{\noindent\textit{Proof sketch.}\hspace*{0.1em}}{\qed\bigskip}

\newcommand\splot[1]{\includegraphics[width=4mm]{figs/img_#1}}

\usepackage{listings}

\lstdefinestyle{python}{
  belowcaptionskip=1\baselineskip,
  breaklines=true,
  frame=shadowbox,
  rulesepcolor=\color{gray},
  xleftmargin=\parindent,
  language=Python,
  showstringspaces=false,
  basicstyle=\footnotesize\ttfamily,
  keywordstyle=\bfseries\color{deepblue},
  moredelim=**[s][\color{blue}]{'''}{'''},
  commentstyle=\itshape\color{magenta},
  identifierstyle=\color{black},
  stringstyle=\color{red},
}

\lstdefinestyle{output}{
  belowcaptionskip=1\baselineskip,
  breaklines=true,
  frame=L,
  basicstyle=\footnotesize\ttfamily,
  xleftmargin=\parindent,
}

\usepackage[normalem]{ulem}

\IEEEoverridecommandlockouts

\begin{document}

\title{Robust Adversarial Learning via\\ Sparsifying Front Ends}

\author{Soorya~Gopalakrishnan\IEEEauthorrefmark{1},
Zhinus~Marzi\IEEEauthorrefmark{1},
Metehan~Cekic\IEEEauthorrefmark{2},
Upamanyu~Madhow\IEEEauthorrefmark{2},
Ramtin~Pedarsani\IEEEauthorrefmark{2}
\thanks{
\IEEEauthorrefmark{1}Qualcomm, Inc. Work done while at Department of Electrical and Computer Engineering, University of California, Santa Barbara. 

\IEEEauthorrefmark{2}Department of Electrical and Computer Engineering, University of California, Santa Barbara.

Emails: \{soorya197, zhinus\_marzi\}@gmail.com, \{metehancekic, madhow, ramtin\}@ucsb.edu}}

\maketitle


\begin{abstract}%
It is by now well-known that small adversarial perturbations can induce classification errors in deep neural networks. 
In this paper, we take a bottom-up signal processing perspective to this problem and show that a systematic exploitation of \textit{sparsity} 
in natural data is a promising tool for defense.  
For linear classifiers, we show that a sparsifying front end is provably effective against $\ell_{\infty}$-bounded attacks, 
reducing output distortion due to the attack by a factor of roughly $K/N$ where $N$ is the data dimension and $K$ is the sparsity level.
We then extend this concept to deep networks, showing that a ``locally linear'' model can be used to develop a theoretical foundation for crafting attacks and defenses. 
We also devise attacks based on the locally linear model that outperform the well-known FGSM attack. 
 We supplement our theoretical results with experiments on the MNIST and CIFAR-10 datasets, showing the efficacy of the proposed sparsity-based defense schemes.
\end{abstract}


\section{Introduction}

Since \citet{szegedy2013intriguing} and \citet{goodfellow2014adversarial} pointed out the vulnerability of deep networks to small {\it adversarial} perturbations, there has been an explosion of research effort in adversarial attacks and defenses \cite{fawzi2017review}. In this paper, we attempt to provide fundamental insight into both the vulnerability of deep networks to carefully designed perturbations, and a systematic, theoretically justified framework for designing defenses against adversarial perturbations.

Our starting point is the original intuition in \citet{goodfellow2014adversarial} that deep networks are vulnerable to small perturbations 
not because of their complex, nonlinear structure, but because of their being ``too linear''.  Consider the simple example
of a binary linear classifier $\bw$ operating on $N$-dimensional input $\bx$, producing the decision statistic $g( \bx ) = \bw^T \bx$.
The effect of a perturbation $\be$ to the input is given by $g( \bx + \be ) - g( \bx ) = \bw^T \be$.  If $\be$ is bounded by
$\ell_{\infty}$ norm (i.e., $\left\Vert \be \right\Vert_{\infty} = \max_i |e_i| \leq \epsilon$), then the largest perturbation that can be produced at the output is caused by
$\be = \epsilon\sign ( \bw )$, and the resulting output perturbation is  $\Delta = \epsilon \sum_{i=1}^N |w_i| = \epsilon \left\Vert \bw \right\Vert_1$.  The latter can be
made large unless the $\ell_1$ norm of $\bw$ is constrained in some fashion.  To see what happens without such a constraint, suppose that $\bw$
has independent and identically distributed components, with bounded expected value and variance. It is easy to see
that $\left\Vert \bw \right\Vert_1 = \Theta(N)$\footnote{See Subsection \ref{notation} for the definitions of the order notation $\Theta(.)$, $\mathcal{O}(.)$, $\omega(.)$, $\littleo(.)$.} with high probability,
 which means that the effect of $\ell_{\infty}$-bounded input perturbations can be blown up at the output as the input dimension increases.

The preceding linear model provides a remarkably good approximation for today's deep CNNs.  Convolutions, subsampling, and average pooling are inherently linear. 
A ReLU unit is piecewise linear, switching between slopes at the bias point.   
A max-pooling unit is a switch between multiple linear transformations.
Thus, the overall transfer function between the input and an output neuron can be written as $\bweq ( \bx )^T \bx$, where $\bweq (\bx)$ is an equivalent ``locally linear'' transformation that
exhibits input dependence through the switches corresponding to the ReLU and max pooling units. For small perturbations, relatively few switches flip, so that $\bweq ( \bx + \be ) \approx \bweq ( \bx )$.
Note that the preceding argument also applies to more general classes of nonlinearities: the locally linear approximation is even better for sigmoids, since slope changes are gradual rather than drastic. These observations motivate us to begin with a study of linear classifiers, before extending our results to neural networks via a locally linear model. 

As we discuss in Section \ref{sec:related}, 
the current state-of-the-art defense is based on retraining networks with adversarial examples \cite{madry2017towards}. However, the lack of insight into what it does fundamentally limits how much we can trust the resulting networks: the sole means of verifying robustness is through empirical evaluations. In contrast, our goal is to provide a systematic approach with theoretical guarantees, with the potential of yielding a longer-term solution to the design of robust neural networks. We take a bottom-up signal processing perspective and
exploit the rather general observation that input data must be sparse in {\it some} basis in order to avoid the curse of dimensionality.
Sparsity is an intuitively plausible concept: we understand that humans reject small perturbations by focusing on the key features that stand out. 
Our proposed approach is based on this intuition.
 In this paper, we show via both theoretical results and experiments that a sparsity-based defense is provably effective against $\ell_\infty$-bounded adversarial perturbations. 

Specifically, we assume that the $N$-dimensional input data has a $K$-sparse representation (where $K\ll N$) in a known orthonormal basis. We then employ a sparsifying front end that projects the perturbed data onto the $K$-dimensional subspace. The intuition behind why this can help is quite clear: small perturbations can add up to a large output distortion when the input dimension is large, and by projecting to a smaller subspace, we can limit the damage. Theoretical studies show that this attenuates the impact of $\ell_{\infty}$-bounded attacks by a factor of roughly $K/N$ (the sparsity level), and experiments show that sparsity levels of the order of 1-5\% give an excellent tradeoff between the accuracy of input representation and rejection of small adversarial perturbations.

\subsection{Contributions}

We develop a theoretical framework to assess and demonstrate the effectiveness of a sparsity-based defense against adversarial attacks on linear classifiers and neural networks. Our main contributions are as follows: 

\begin{itemize}[itemsep=2pt, topsep=2pt, leftmargin=15pt,]
\item For linear classifiers, we quantify the achievable gain of the sparsity-based defense via an ensemble-averaged analysis based on a stochastic model for weights. We prove that with high probability, the adversarial impact is reduced by a factor of roughly $K/N$, where $K$ is the sparsity of the signal and $N$ is the signal's dimension.

\item For neural networks, we use a ``locally linear'' model to provide a framework for understanding the impact of small perturbations. Specifically, we characterize a high SNR regime in which the fraction of switches flipping for ReLU nonlinearities for an $\ell_{\infty}$-bounded perturbation is small.

\item Using the preceding framework, we show that a sparsifying front end is effective for defense, and devise a new attack based on locally linear modeling.

 \item We supplement our theoretical results with experiments on the MNIST \cite{lecun1998gradient} and CIFAR-10 \cite{krizhevsky2009learning} datasets for a variety of recent attacks.

\end{itemize}

\subsection{Notation} \label{notation}

Here we define the notations $\Theta(.)$, $\mathcal{O}(.)$, $\omega(.)$, and $\littleo(.)$:
\begin{itemize}[leftmargin=20pt]
\item $f = \mathcal{O}(g)$ if and only if there exists a constant $C > 0$ such that $|f/g| < C$,
\item $f = \Theta(g)$ if and only if there exist two constants $C_1 , C_2 > 0$ such that $C_1 < |f/g| < C_2$,
\item $f = \omega(g)$ if and only if there does not exist a constant $C > 0$ such that $|f/g| < C$, and
\item $f = \littleo(g)$ if and only if $g = \omega(f)$.
\end{itemize}


\section{Related Work} \label{sec:related}

\subsection{Sparse Representations}

It is well-known that most natural data can be compactly expressed as a sparse linear combination of atoms in \textit{some} basis \cite{mallat1993matching,blumensath2008gradient,bruckstein2009sparse}. Such sparse representations have led to state-of-the-art results in many fundamental signal processing tasks such as image denoising \cite{donoho1994ideal,aharon2006rm}, neuromagnetic imaging \cite{gorodnitsky1997sparse,wright2009sparse}, image super-resolution \cite{yang2010image}, inpainting \cite{elad2005simultaneous}, blind audio source separation \cite{zibulevsky2001blind}, source localization \cite{malioutov2005sparse}, etc. There are two broad classes of sparsifying dictionaries employed in literature: predetermined bases such as wavelets \cite{mallat1993matching,mallat1999wavelet}, and learnt bases which are inferred from a set of training signals \cite{aharon2006rm, skretting2010recursive, rubinstein2013analysis}. 
The latter approach has been found to outperform predetermined dictionaries.
Sparsity has also been suggested purely as a means of improving classification performance \cite{makhzani2013ksparse}, which indicates that the performance penalty for appropriately designed sparsity-based defenses could be minimal.

\subsection{Local Linearity Hypothesis}
The existence of ``blind spots'' in deep neural networks \cite{szegedy2013intriguing} has been the subject of extensive study in machine learning literature \cite{fawzi2017review}.
It was initially surmised that this vulnerability is due to the highly complex, nonlinear nature of neural networks. However, the success of linearization-based attacks such as the Fast Gradient Sign Method (FGSM) \citep{goodfellow2014adversarial} and DeepFool \citep{moosavi2016deepfool} indicates that it is instead due to their ``excessive linearity''. This is further backed up by work on the curvature profile of deep neural networks \citep{fawzi2017classification,ben2016expressivity} showing that the decision boundaries in the vicinity of natural data can be approximated as flat along most directions. Our work on attack and defense is grounded in a locally linear model for neural networks, and the results discussed later demonstrate the efficacy of this approach.

\subsection{Defense Techniques}

Existing defense mechanisms against adversarial attacks can be broadly divided into two categories: (a) empirical defenses, based purely on intuitively plausible strategies, together with experimental evaluations, and (b) provable defenses with theoretical guarantees of robustness. We briefly discuss the empirical work related to our defense, and then give an overview of provable defense techniques.

\subsubsection*{Empirical Defenses}

One approach to defense is to retrain networks with adversarial examples, which works well if the examples are generated using multi-step attacks such as PGD \cite{madry2017towards}. However, this incurs a large computational cost.
A more fundamental drawback of such an approach is the lack of guarantees of robustness, in contrast to our proposed method.
Several empirical defenses utilize preprocessing techniques that are implicitly sparsity-based, including JPEG compression \citep{guo2017inputtransform,das2017keeping}, PCA \citep{bhagoji2017dimensionality} and projection onto generative models \citep{ilyas2017robust,samangouei2018defensegan}. However many such techniques were found to confer robustness purely by obfuscating gradients necessary for the adversary, and they were successfully defeated by gradient approximation techniques developed in \cite{obfuscated-gradients}. In contrast our sparsity-based defense is robust to the attack techniques in \cite{obfuscated-gradients}. 
Overall, the evaluations in such prior work have been purely empirical in nature. 
Our proposed framework, grounded in a locally linear model for neural networks, provides a foundation for systematic pursuit of sparsity-based preprocessing.

\subsubsection*{Provable Defenses}

There is a growing body of work focused on developing \textit{provable} guarantees of robustness against adversarial perturbations \cite{raghunathan2018certified,raghunathan2018semidefinite,kolter2017provable,wong2018scaling,sinha2018certifiable,mirman2018differentiable,hein2017formal,cisse2017parseval,isit_2018,bafna2018thwarting,fawzi2018adversarial,katz2017reluplex,ehlers2017formal,cheng2017maximum,dutta2018output}. We do not provide an exhaustive discussion due to space constraints, but we illustrate some typical characteristics of these approaches.

Many provable defenses focus on retraining the network with an optimization criterion that promotes robustness towards all possible small perturbations around the training data.
For example, \citet{raghunathan2018certified} develop semidefinite programs (SDP) to upper bound the $\ell_1$ norm of the gradient of the classifier around the data point, which they then optimize during training to obtain a more robust network. However, the certificates provided are quite loose.
 A tighter SDP relaxation is developed in \citep{raghunathan2018semidefinite} to certify the robustness of any given network, but it is not used to train a new robust model. Both the SDP certificates are computationally restricted to fully-connected networks.
Another provable defense by \citet{kolter2017provable} employs a linear programming (LP) based approach to bound the robust error. 
This is more efficient than the SDP approach, and can be scaled to larger networks \cite{wong2018scaling}. The LP technique works well for small perturbations, but for larger perturbation sizes there is a drop in the accuracy of natural images. 

A different training technique based on distributionally robust optimization was proposed by \citet{sinha2018certifiable}, providing a certificate of robustness for attacks whose probability distribution is bounded in Wasserstein distance from the original data distribution. Although the theoretical guarantees do not directly translate to norm-bounded attacks, this approach yields good results in practice for small perturbation sizes.
Another defense that works well is that of \citet{mirman2018differentiable}, who use the framework of abstract interpretation to develop various upper bounds on the adversarial loss which can be optimized over during training. 
Some other provable defenses try to limit Lipschitz constants related to the network output function, for example via cross-Lipschitz regularization \citep{hein2017formal} and $\ell_2$ matrix-norm regularization of weights \citep{cisse2017parseval} in order to defend against $\ell_2$-bounded attacks.
However the bounds in \citep{hein2017formal} are limited to 2-layer networks, and \citep{cisse2017parseval} is based on layerwise bounds that have been shown to be loose \citep{raghunathan2018certified,kolter2017provable}. 

All of the above techniques try to train the network so as to make its output less sensitive to input perturbations by modifying the optimization framework employed for network training.  
In contrast, our defense takes a bottom-up signal processing approach, exploiting the sparsity of natural data to combat perturbations at the front end,
while using conventional network training.  It is therefore potentially complementary to defenses based on modifying network training.
Parts of this work appeared previously in our conference paper \cite{isit_2018} which focused on linear classifiers. In this paper, we provide a comprehensive treatment that applies to neural networks, by employing the concept of local linearity.
Another related work that uses sparsity-based preprocessing is \citep{bafna2018thwarting}, which studies $\ell_0$-bounded attacks. These are not visually imperceptible (unlike the $\ell_\infty$-bounded attacks that we study), but they are easy to realize in practice, for example by placing a small sticker on an image. For this class of attacks, \citep{bafna2018thwarting} shows that the sparse projection can be reformulated as a compressive sensing (CS) problem, and obtain a provably good estimate of the original image via CS recovery algorithms. 

There is also a line of work on \textit{verifying} the robustness of a given network by using exact solvers based on discrete optimization techniques such as satisfiability modulo theory \cite{katz2017reluplex,ehlers2017formal} and mixed-integer programming \cite{cheng2017maximum,dutta2018output}. However these techniques have combinatorial complexity (in the worst-case, exponential in network size), and so far have not been scaled beyond moderate network sizes.
\nocite{carlini2016distillationrefuted}


\section{Sparsity Based Defense} \label{sec:sparsity_defense}

\subsection{Problem Setup}

For simplicity we start with binary classification. Given a binary inference model $g: \mathbb{R}^{N} \to \mathbb{R}$, we assume that its input data $\bx \in \mathbb{R}^{N}$ has a $K$-sparse representation ($K\ll N$) in a known orthonormal basis $\Psi$: $\left\Vert \Psi^T\bx\right\Vert_0 \le K$. Let us denote by $\bxhat$ a {\it modified} version of the data sample $\bx$. We now define a performance measure $\Delta$ that quantifies the robustness of $g(\cdot)$:
\begin{equation*}
\Delta \left( \bx, \bxhat\right) = |g(\bxhat) - g(\bx)|.
\end{equation*}
For example, for a linear classifier $g(\bx) = \bw^T \bx$, the performance measure is $\Delta \left( \bx, \bxhat\right) = |\bw^T(\bxhat - \bx)|$.

We now consider a system comprised of the classifier $g(\cdot)$ and two external participants: the adversary and the defense, depicted in Fig. \ref{fig:main_block}.
\begin{enumerate}[label=\arabic*.,leftmargin=18pt, itemsep=3pt]
\item 
The adversary corrupts the input $\bx$ by adding a perturbation $\be$, with the goal of causing misclassification:
\begin{equation*}
\max_{\be} \,{\Delta \left( \bx, \bxhat \right)} \quad \mathrm{s.t.}
\quad \left\Vert\be\right\Vert_\infty<\epsilon.
\end{equation*}
We are interested in perturbations that are visually imperceptible, hence we impose an $\ell_\infty$-constraint on the adversary.  
\item The defense preprocesses the perturbed data via a function $f: \mathbb{R^N} \to \mathbb{R^N}$, with the goal of minimizing the adversarial impact $\Delta$. 
\end{enumerate}
We consider two threat models: a semi-white box scenario where the perturbation is designed based on knowledge of the classifier alone, and a white box scenario where the adversary has complete knowledge of both the defense and classifier. This setup can be easily extended to multiclass classification as we show in Section V-E. For multiple classes, the goal of the adversary is to cause untargeted misclassification, i.e.\ shift the output of the classifier to \textit{any} incorrect class.

\begin{figure}[!t]
  	\centering
	\begin{tikzpicture}[
	rect/.style = {rectangle, draw=black!100, fill=cyan!18, thin, minimum height=7.5mm, minimum width=11mm},
	rect_smaller/.style = {rectangle, draw=black!100, fill=cyan!18, thin, minimum height=7.5mm, minimum width=8mm},
	circ/.style = {circle, draw=black!100, fill=cyan!18, thin, minimum size=7mm},
	outer/.style = {rounded corners=0.2cm, draw=black!70, dashed, inner sep = 2.3mm}
	]
	\node[] (x) {$\bx$};
	\node[circ] (plus) [right = 7mm of x] {$+$};
	\node[] (e) [above = 6mm of plus] {$\be$};
	\node[rect] (f) [right = 13mm of plus] {$f(\cdot)$};
	\node[rect] (w) [right = 12mm of f] {$g(\cdot)$};
	\node[outer, inner sep = 2.2mm, fit = (plus), label=below:\footnotesize{Adversary}] (attack) {};
	\node[outer,fit = (f), label=below:\footnotesize{Defense}] (frontend) {};
	\draw[->] (x) -- (plus);
	\draw[->] (e) -- (plus); 
	\draw[->] (plus) -- node[anchor=south]{$\bxbar$} (f); 
	\draw[->] (f) -- node[anchor=south]{$\bxhat$}(w);
	\end{tikzpicture}
	\caption{Block diagram of the system, depicting an adversarial example $\bxbar = \bx + \be$ (with $\ell_\infty$ constraint on $\be$), a preprocessing defense $f(\cdot)$ and a classifier $g(\cdot)$.}
  	\label{fig:main_block}
\end{figure}
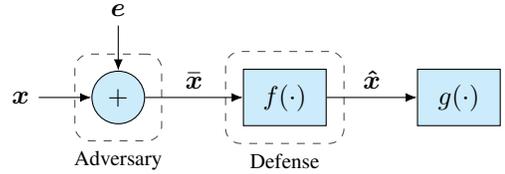

\begin{figure}[!b]
  	\centering
	\begin{tikzpicture}[
	rect/.style = {rectangle, draw=black!100, fill=cyan!18, thin, minimum height=7.5mm, minimum width=11mm, 
	inner sep=2mm},
	rect2/.style = {rectangle, draw=black!100, fill=cyan!1, thin, minimum height=7.5mm, minimum width=11mm, 
	inner sep=1.5mm},
	rect_smaller/.style = {rectangle, draw=black!100, fill=cyan!18, thin, minimum height=7.5mm, minimum width=8mm},
	circ/.style = {circle, draw=black!100, fill=cyan!18, thin, minimum size=7mm},
	outer/.style = {rounded corners=0.2cm, draw=black!100, dashed, inner sep = 2mm}
	]	
	\node[] (x) {$\bx$};
	\node[circ] (plus) [right =4mm of x] {$+$};
	\node[] (e) [above = 4mm of plus] {$\be$};
	\node[rect_smaller] (f) [right = 7mm of plus, align=center] {$\Psi^T$};
	\node[rect] (topk) [right = 3mm of f, align=center] {\footnotesize{Retain} \\[-1pt] \footnotesize{$K$ largest} \\[-1pt] \footnotesize coefficients};
	\node[rect_smaller] (ft) [right = 3mm of topk, align=center] {$\Psi$};
	\node[rect2] (w) [right = 7mm of ft] {\includegraphics[width=.08\textwidth]{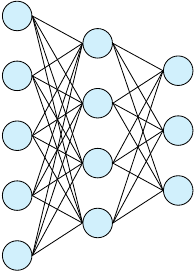}};
	\node[outer, fit = (f) (ft), label=below:\small{Sparsifying front end}, minimum height=22mm] (frontend) {};
	\node[fit = (w), label=below:\huge{$\phantom{df}$}] (network) {};
	\draw[->] (x) -- (plus);
	\draw[->] (e) -- (plus); 
	\draw[->] (f) -- node[anchor=south]{}(topk);
	\draw[->] (topk) -- node[anchor=south]{}(ft);
	\draw[->] (plus) -- node[anchor=south]{$\bxbar\phantom{d}$} (f); 
	\draw[->] (ft) -- node[anchor=south]{$\phantom{d}\bxhat$}(w);
	\end{tikzpicture}
\caption{Sparsifying front end defense: For a basis in which the input is sparse, the input is projected onto the subspace spanned by the $K$ largest basis coefficients.}
\label{fig:front_end}
\end{figure}
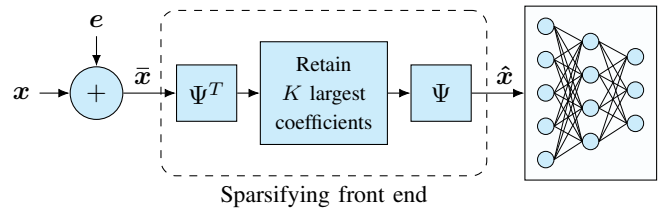

\subsection{Sparsifying Front End Description}

We propose a defense based on a \textit{sparsifying front end} that exploits sparsity in natural data to combat adversarial attacks. Specifically, we preprocess the input via a front end that computes a $K$-sparse projection in the basis $\Psi$. Figure \ref{fig:front_end} shows a block diagram of our model for a neural network, depicting an additive perturbation followed by sparsity-based preprocessing. 

Here is an intuitive explanation of how this defense limits adversarial perturbations. 
If the data is exactly $K$-sparse in domain $\Psi$, the front end leaves the input unchanged ($\bxhat=\bx$) when there is no attack ($\be=0$). The front end attenuates the perturbation by projecting it onto the space spanned by the basis functions corresponding to the $K$ retained coefficients. If the perturbation is small enough, then the $K$ retained coefficients corresponding to $\bx$ and $\bx + \be$ remain the same, in which case the neural network sees the original input plus the projected, and hence attenuated, perturbation.

Let $\sparse: \mathbb{R}^N \to \mathbb{R}^N$ represent the block that enforces sparsity by retaining the $K$ coefficients largest in magnitude and zeroing out the rest. We can now define the following:
\begin{itemize}[leftmargin=15pt, itemsep=3pt, topsep=3pt]
\item The support $\support$ of the $K$-sparse representation of $\bx$, and the projection $\proj$ of $\be$ onto the subspace that $\bx$ lies in are defined as follows:
	\begin{align*}
		\support \left(\bx\right) &\triangleq \supp \left(\sparse \left(\Psi^T\bx\right)\right), \\[6pt]  \proj \left(\be,\bx\right) &\triangleq  \sum_{k \in \support(\bx)} \bpsi_k \bpsi_k^T\be.
	\end{align*} 
\item The {\it high SNR regime} is the operating region where the perturbation does not change the subspace that $\bx$ lies in:
\begin{equation}  \label{eq:Complete_recovery}
\textrm{High SNR:}\quad \support(\bx)=\support(\bx+\be).
\end{equation}
We characterize the conditions guaranteeing (\ref{eq:Complete_recovery}) in Prop.\ \ref{prop:SNR_condition}.
\end{itemize}\vspace{5pt}

\noindent If we operate at high SNR, the front end preserves the signal
and hence its output $\bxhat$ can be written as follows:
\begin{equation*}
\bxhat = \bx + \sum_{k \in \support(\bx)}\bpsi_k \bpsi_k^T \be = \bx +  \proj(\be,\bx).
\end{equation*}
Thus the effective perturbation is $\proj(\be,\bx)$, which lives in a lower dimensional space. Its impact is therefore significantly reduced.  In Sections \ref{sec:linear_classifier} and \ref{sec:neural_nets}, we quantify the reduction in adversarial impact via an ensemble-averaged analysis based on a stochastic model for the classifier $g(\cdot)$. 

\subsection{Characterizing the High SNR Regime} \label{sec:high_snr}

\noindent We can gain valuable design insight by characterizing the conditions that guarantee high SNR operation of the sparsifying front end, as stated in the following proposition (with proof in Appendix A): 

\begin{proposition} \label{prop:SNR_condition}
For sparsity level $K$ and perturbation $\be$ with $\Vert\be\Vert_\infty\leq\epsilon$, the sparsifying front end preserves the input coefficients if the following SNR condition holds:
\begin{equation*}
\mathrm{SNR} \, \triangleq \, \frac{\lambda}{\epsilon} \,> \gamma,
\end{equation*}
where $\lambda$ is the magnitude of the smallest non-zero entry of $\sparse(\Psi^T \bx)$ and $\gamma = 2\,\max_{k}{\left\Vert \bpsi_k\right\Vert_1}$.
\end{proposition}\vspace{1pt}


A direct consequence of the SNR condition is that we expect basis functions that are small in $\ell_1$ norm to be more effective. As we will see later in Section \ref{sec:linear_performance}, this is also a favorable criterion for performance in the white box attack scenario. Another important design parameter is the value of $K$, which must be chosen to optimize the following tradeoff: lower sparsity levels allow us to impose high SNR even for larger perturbations, but if the data is only approximately $K$-sparse, this results in unwanted signal perturbation. We find in our experiments that a choice of $K/N$ of the order of 1--5\% provides an excellent balance to this tradeoff. 

Our subsequent analysis assumes that the SNR condition holds, in which case we can quantify the reduction in adversarial impact solely via the effect of the front end on the perturbation. We start with a study of linear classifiers and then extend our results to neural networks via a locally linear model.


\section{Analysis for Linear Classifiers} \label{sec:linear_classifier}

Consider a linear classifier $g(\bx) = \bw^T \bx$.  We calculate the adversarial impact for various attack models and quantify the efficacy of our defense by using a stochastic model for $\bw$. 

\subsection{Impact of Adversarial Perturbation}

When the front end is not present, the impact of the attack is $\Delta=\bw^T \be$. By Holder inequality, we have 
\begin{equation} \label{holder_no_def}
\Delta=\bw^T \be \leq \Vert \be \Vert_\infty \Vert \bw \Vert_1 \leq \epsilon \,\Vert \bw \Vert_1 \triangleq \Delta_0,
\end{equation} 
where the second inequality follows from the $\ell_\infty$ attack budget constraint. We can observe that $\be_0=\epsilon \,\sign(\bw)$ achieves equality in (\ref{holder_no_def}), which means that $\be_0$ is the optimal attack when the adversary has knowledge of $\bw$. We use $\Delta_0 = \epsilon \Vert \bw \Vert_1$ as a baseline to assess the efficacy of our defense.

When the defense is present, the adversarial impact $\Delta$ becomes
\begin{equation}
\Delta = \left| \bw^T (\bxhat - \bx) \right| = \left| \bw^T \proj(\be,\bx) \right|
= \left| \be^T \proj(\bw,\bx) \right|
\end{equation}
where the second equality follows from the definition of $\proj(\be,\bx)$. 
We can now consider two scenarios depending on the adversary's knowledge of the defense and the classifier:

\begin{enumerate}[itemsep=5pt, topsep=5pt, label=\arabic*.,leftmargin=15pt]

\item\textit{Semi-white box scenario:} Here perturbations are designed based on the knowledge of $\bw$ alone, and therefore the attack remains $\be_{\mathrm{SW}}=\epsilon \sign \left(\bw\right)$. The output distortion becomes
\begin{equation} \label{eq:delta_sw}
\Delta_{\mathrm{SW}} = \epsilon\, \left| \sign(\bw^T) \proj(\bw,\bx) \right|.
\end{equation}
We note that the attack is aligned with $\bw$. 
\item\textit{White box scenario:} Here the adversary has the knowledge of both $\bw$ and the front end, and designs perturbations in order to maximize $\Delta=\left| \be^T \proj(\bw,\bx) \right|$.
We can use the same Holder inequality based argument as before to prove that the optimal perturbation is $\be_{\mathrm{W}}=\epsilon \sign \left(\proj \left(\bw,\bx \right)\right)$. 
The resulting output distortion can be written as
\begin{equation*}
\Delta_{\mathrm{W}}=\epsilon \, \left\Vert\proj (\bw,\bx)\right\Vert_1.
\end{equation*}
Thus, instead of being aligned with $\bw$, $\be_{\mathrm{W}}$ is aligned to the projection of $\bw$ onto the subspace that $\bx$ lies in. 
\end{enumerate} 

\subsection{Ensemble Averaged Performance} \label{sec:linear_performance}

We now quantify the gain in robustness conferred by the sparsifying front end by taking an ensemble average over randomly chosen classifiers $\bw$.

\begin{assumption}
We assume a random model for $\bw$, where the entries $\{ w_i\}_{i=1}^N$ are i.i.d.\ with zero mean and median. Let $\E\left[|w_1|\right] = \mu = \Theta(1)$ and $\E\left[w_1^2\right] = \sigma^2 = \Theta(1)$.
\end{assumption}

\noindent We observe that the baseline adversarial impact $\Delta_0 = \epsilon\, \Vert \bw\Vert_1 $ scales with $N$. This is formalized in the following proposition:
\begin{proposition}
$\Delta_0/N$ converges to $\epsilon\,\mu$ almost surely, i.e\
\begin{equation*}
\prob \left( \lim_{N\to\infty} \frac{\Delta_{0}}{N} = \epsilon\,\mu \right) = 1.
\end{equation*}
Thus, with no defense, the adversarial impact scales as $\Theta(N)$.
\end{proposition}
\begin{proof}
$\Delta_0$ is the sum of i.i.d.\ random variables $\epsilon\, |w_i|$ with finite mean: $\E[ \epsilon\, |w_i|] = \epsilon\, \mu<\infty$. Hence we can apply the strong law of large numbers, which completes the proof.
\end{proof}

We now state the following theorems that characterize the performance of the sparsifying front end defense in the semi-white box and white box scenarios.

\vspace{5pt}
\noindent1. \textit{Semi-White Box Scenario:}

\begin{theorem} \label{theorem1}
As $K$ and $N$ approach infinity, $\Delta_{\mathrm{SW}}/K$ converges to $\epsilon\,\mu$ in probability,
i.e.\
\begin{equation*}
\lim_{K\to\infty} \prob \left( \left| \frac{\Delta_{\mathrm{SW}}}{K} - \epsilon\,\mu \right| \leq \delta \right) = 1 \quad \forall \; \delta >0.
\end{equation*}
Thus, the impact of adversarial perturbation in the case of semi-white box attack is attenuated by a factor of $K/N$ compared to having no defense.
\end{theorem}

\noindent The proof can be found in Appendix A.\vspace{3pt}

\vspace{5pt}
\noindent2. \textit{White Box Scenario:}

\vspace{5pt}
\noindent We begin with the following lemma that provides a useful upper bound on the impact of the white box attack.
\begin{lemma} \label{lemma:triangle}
An upper bound on the white box distortion $\Delta_{\mathrm{W}}$ is given by 
\begin{equation*}
\Delta_{\mathrm{W}} \leq \epsilon\, \sum_{k=1}^K \left| \bpsi_k^T \bw \right| \left\Vert \bpsi_k \right\Vert_1.
\end{equation*}
\end{lemma}
\begin{proof}
\begin{align*}
\Delta_{\mathrm{W}} &= \epsilon \left\Vert\proj \left(\bw,\bx \right)\right\Vert_1 
= \epsilon \,\sum_{i=1}^N \, \left| \sum_{k=1}^K \left( \bpsi_k^T \bw \right) \psi_k\left[i\right] \right| 
\\&\leq \epsilon \,\sum_{i=1}^N \sum_{k=1}^K \left| \bpsi_k^T \bw \right| \left| \,\psi_k\left[i\right] \right|
= \epsilon \,\sum_{k=1}^K  \left| \bpsi_k^T \bw \right|  \left\Vert \bpsi_k \right\Vert_1.
\end{align*}
\end{proof}
\vspace{-5pt}
\noindent Note that this bound is exact if the supports of the $K$ selected basis functions do not overlap, and consequently the white box distortion cannot grow slower than $K$ since the bound has $K$ terms. However, if the $\ell_1$ norms of the basis functions do not scale too fast with $N$, we can show that the distortion scales as $\bigo\left(K \polylog(N)\right)$, as stated in the following theorem.

\begin{theorem} \label{theorem2}
Under the assumptions $\left\Vert \bpsi_k \right\Vert _1 = \bigo(\log N)$, $\left\Vert \bpsi_k \right\Vert _\infty = \littleo(1)$ $\,\forall\, k\in\{1,2,\dots,K\}$, and $\left\Vert \bw \right\Vert _\infty = \bigo(1)$, 
we have the following upper bound for $\Delta_\mathrm{W}$:
\begin{equation*}
\lim_{N \to \infty} \Pr\left(\Delta_\mathrm{W} \leq \bigo\left(K \polylog(N)\right)\right)=1.
\end{equation*}
Thus, the impact of adversarial perturbation in the case of white box attack is attenuated by a factor of $\bigo(K \polylog(N)/N)$ compared to having no defense.
\end{theorem}
\begin{proof} \label{white_proof}
We first state the following convergence result (see Appendix A for the proof):
\begin{lemma} \label{lemma:CLT}
$\bpsi_k^T \bw \rightarrow \mathcal{N}(0,\,\sigma^2)$ in distribution.
\end{lemma}

\noindent We can use Lemmas \ref{lemma:triangle} and \ref{lemma:CLT} to obtain
\begin{align*}
\Pr&\left(\Delta_\mathrm{W} > \delta\right) \leq \Pr \left( \epsilon\,\sum_{k=1}^K \left|\bpsi_k^T \bw\right|\, \left\Vert\bpsi_k\right\Vert_1 > \delta \right) 
\\ &\leq \Pr \left( \bigcup_{k=1}^K \left\{\left|\bpsi_k^T \bw\right| > \frac{\delta}{\epsilon\,K \left\Vert\bpsi_k\right\Vert_1} \right\} \right) 
\\ &\leq \sum_{k=1}^K \Pr \left( \left|\bpsi_k^T \bw\right| > \frac{\delta}{\epsilon\,K \left\Vert\bpsi_k\right\Vert_1} \right) 
\\ &= \sum_{k=1}^K 2 Q\left(\frac{\delta}{\epsilon\,\sigma K \left\Vert\bpsi_k\right\Vert_1}\right)
= 2K Q\left(\frac{\delta}{\epsilon\,\sigma} \bigo\left(\frac{1}{K \log N}\right)\right),
\end{align*}
where the last step follows from the $\ell_1$ assumption on $\bpsi_k$, and $Q(x)$ is the Gaussian tail distribution function $\int_{x}^{\infty}  e^{-t^2/2}dt \,/\sqrt{2 \pi}$. We complete the proof by setting $\delta = \bigo(K \polylog(N))$ which makes the right-hand side of the above equation vanish as $N$ approaches infinity.
\end{proof}

\noindent 
A practical take-away from  the above theoretical results is that, in order for the defense to be effective against a white box attack, not only do we need input sparsity ($K \ll N$), but we also need that the individual basis functions be localized (small in $\ell_1$ norm). The latter implies, for example, that sparsification with respect to a wavelet basis, which has more localized basis functions, should be more effective than with a DCT basis.


\section{Analysis for Neural Networks} \label{sec:neural_nets}

In this section, we build on the preceeding insights for neural networks by exploiting locally linear approximations. For simplicity we start with a 2-layer, fully connected network trained for binary classification, and then extend our results to a general network for multi-class classification.

\subsection{Locally Linear Representation}

Consider first binary classification using a neural network with one hidden layer with $M$ neurons, as depicted in Fig. \ref{fig:2-layer-nn}. Since ReLU units are piecewise linear, switching between slopes of 0 and 1, they can be represented using input-dependent switches. Given input $\bx$, we denote by $s_i(\bx) \in\{0,1\}$ the switch corresponding to the $i$th ReLU unit. Now the activations of the hidden layer neurons can be written as follows:
\begin{equation*}
a_i =  \mathrm{ReLU}\left(\bw_i^T \bx - b_i \right) = s_i(\bx)\, \bw_i^T \bx - s_i(\bx)\, b_i.
\end{equation*}
The output of the neural network can be written as
\begin{align*}
y(\bx) = \bw_0^T \ba &= \sum_{i=1}^M s_i(\bx)\, w_0[i] \,\bw_i^T \bx - \sum_{i=1}^M s_i(\bx)\, w_0[i] \, b_i 
\\
&
 = \bweq(\bx)^T \bx - \, \beq(\bx).
\end{align*}
where 
\begin{align*} 
\bweq(\bx) 
= \sum_{i=1}^M \, s_i(\bx) \, w_0[i] \,\bw_i, 
\quad
\beq(\bx) 
= \sum_{i=1}^M s_i(\bx) \, w_0[i] \, b_i.
\end{align*}

This locally linear model extends to any standard neural network, since convolutions and subsampling are inherently linear and max-pooling units can also be modeled as switches.  For more than 2 classes, we will apply this modeling approach to the ``transfer function'' from the input to the inputs to the softmax layer, as discussed
in Section \ref{sec:multiclass}.


\begin{figure}[t]
\centering
  	\begin{tikzpicture}[
	rect_h/.style = {rectangle, draw=black!100, fill=cyan!18, thin, minimum height=7.5mm, minimum width=11mm},
	rect_v/.style = {rectangle, draw=black!100, fill=cyan!18, thin, minimum height=12mm, minimum width=7mm},
	rect_v_big/.style = {rectangle, draw=black!100, fill=cyan!18, thin, minimum height=12mm, minimum width=7mm},
	circ/.style = {circle, draw=black!100, fill=cyan!18, thin, minimum size=7mm},
	outer/.style = {rounded corners=0.2cm, draw=black!100, dotted, thick, inner sep = 1.5mm},
	rect/.style = {rectangle, draw=black!100, fill=cyan!18, thin, minimum height=7mm, minimum width=7mm},
	]
	\node[] (x) {$\bx$};
	\node[] (conn) [right = 6mm of x] {};
	\node[rect_v] (w2) [above right = 8mm of conn, align = center] {$\bw_2$\\[5pt]$b_2$};	
	\node[rect_v] (w1) [above = 5mm of w2, align = center] {$\bw_1$\\[5pt]$b_1$};
	\node[rect_v] (wm) [below = 17mm of w2, align = center] {$\bw_\subM$\\[5pt]$b_\subM$};
	\node[] (dots) [above = 5mm of wm] {$\vdots$};
	\node[circ, label=below:\scriptsize{ReLU}] (sig1) [right = 5mm of w1] {\splot{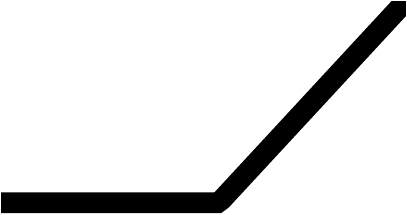}};
	\node[circ] (sig2) [right = 5mm of w2] {\splot{relu}};
	\node[circ] (sigm) [right = 5mm of wm] {\splot{relu}};	
	\node[] (dots2) [right = 8mm of dots] {$\vdots$};
	\node[rect] (w0) [right = 15mm of dots2, align = center] {$\bw_0$};
	\node[circ, label=below:\scriptsize{Sigmoid}] (sig0) [right = 10mm of w0] {\small{\splot{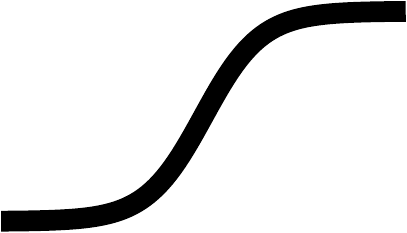}}};
	\node[] (out) [right = 4mm of sig0] {};
	\begin{scope}[decoration={markings,
    mark=at position 0.5 with {\arrow{Latex[width=1.2mm,length=1.8mm]}}}] 
    \draw[postaction={decorate}] (x.north east) -- (w1.west);
    \draw[postaction={decorate}] (x.east) -- (w2.west); 
    \draw[postaction={decorate}] (x.south east) -- (wm.west);
    \draw[postaction={decorate}] (sig1.east) -- node[anchor=220]{$a_1$} (w0.160);
    \draw[postaction={decorate}] (sig2.east) -- node[anchor=60]{$a_2$} (w0.180); 
    \draw[postaction={decorate}] (sigm.east) -- node[anchor=140]{$a_\subM$} (w0.210);    
	\end{scope}
	\begin{scope}[decoration={markings,
    mark=at position 0.7 with {\arrow{Latex[width=1.2mm,length=1.8mm]}}}] 
	\draw[postaction={decorate}] (w0) -- node[anchor=270, xshift=1.5pt]{$y(\bx)$} (sig0);
	\end{scope}
	\draw[->] (w1) -- (sig1);
	\draw[->] (w2) -- (sig2);
	\draw[->] (wm) -- (sigm);
	\draw[->] (sig0) -- (out);
	\node[outer, fit = (w1) (sig1) (wm) (sigm), label=below:\small{Hidden layer}] (n0) {};
	\end{tikzpicture}
	\captionof{figure}{Two layer neural network for binary classification. ReLU units are piecewise linear, hence the network is locally linear: $y(\bx) = \bweq(\bx)^T\,\bx - \beq(\bx)$.}
	\label{fig:2-layer-nn}
\end{figure}
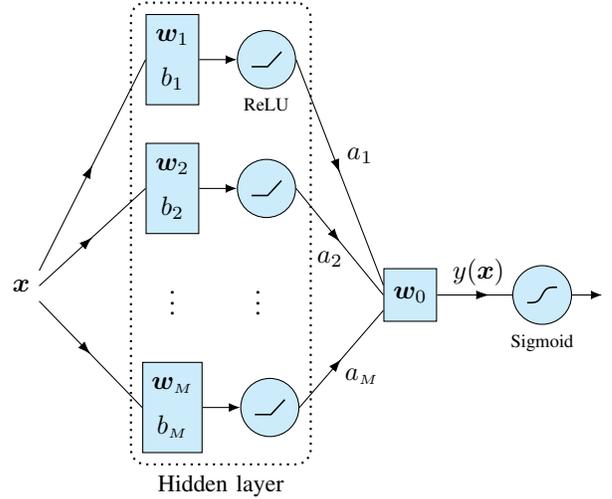

\subsection{Impact of Adversarial Perturbation} \label{subsec:snr}

Now we consider the effect of an $\ell_\infty$-bounded perturbation $\be$ on the performance of the network. For ease of notation, we write $\bweq = \bweq(\bx)$, $\bweqbar = \bweq(\bx+\be)$, and $\beqbar = \beq(\bx + \be ) $. The distortion due to the attack can be written as
\begin{align} \nonumber 
\Delta &= y(\bx + \be) - y(\bx) \\ \nonumber
& = \bweqbar^T (\bx + \be) - \beqbar - \bweq^T \bx + \beq \\ \label{eq:dist}
& =  \bweqbar^T \be + \left[\left(\bweqbar -\bweq \right)^T \bx  - \left(\beqbar - \beq\right)\right] 
\end{align}
We observe that the distortion can be split into two terms: (i) $\bweqbar^T \be$ that is identical to the distortion for a linear classifier, and can be analyzed within the theoretical framework of 
Section \ref{sec:linear_classifier}
 and (ii) $\left(\bweqbar -\bweq \right)^T \bx  - \left(\beqbar - \beq\right)$, that is determined by the ReLU units that flip due to the perturbation. 

In the next section, we provide an analytical characterization of a ``high SNR'' regime 
in which the number of flipped switches is small,
motivated by iterative attacks which gradually build up attack strength over a large number of iterations
(with a per-iteration $\ell_{\infty}$-budget of $\delta \ll \epsilon$).  
When very few switches flip, the distortion is dominated by the first term in \eqref{eq:dist}, and we can apply our prior results on linear classifiers to infer the efficacy of the sparsifying front end in attenuating the distortion. As we discuss via our numerical results, this creates a situation in which it might sometimes be better (depending on dataset and attack budget) for the adversary to try to make the most of network's nonlinearity, spending the attack budget in one go trying to flip a larger number of switches in order to try to maximize the impact of the second term in \eqref{eq:dist}.  

\subsection{Characterizing the High SNR Regime}

We now investigate the conditions that guarantee high SNR at neuron $i$, i.e.\ $\sbar_i = s_i$, where $\sbar_i$ denotes the switch when the adversary is present. 

We observe that 
\begin{align*} 
\sbar_i = \begin{cases}  1-s_i, & \bw_i^T\bx -b_i \, \in \big[\min\big(-\bw_i^T\be,0\big), \,\\
& \phantom{ \bw_i^T\bx -b_i \, \in\big[} \max\big(-\bw_i^T\be,0\big) \big] \\[5pt]
\phantom{-}s_i, & \bw_i^T\bx -b_i \, \notin \big[\min\big(-\bw_i^T\be,0\big), \,\\&\phantom{ \bw_i^T\bx -b_i \, \in\big[}
\max\big(-\bw_i^T\be,0\big) \big],
\end{cases}
\end{align*}
This implies the following sufficient condition for high SNR at neuron $i$:
$\left| \bw^T_i \bx - b_i  \right|  > \left| \bw^T_i \be \right|$.

\vspace{5pt}
\begin{assumptions}\ 
To establish our theoretical result, we make a few mild technical assumptions: 
\begin{enumerate}[label=\arabic*.,leftmargin=15pt, itemsep=2pt, topsep=4pt]

\item The data is normalized in $\ell_2$-norm and bounded: $\left\Vert\bx\right\Vert_2 = 1$ and $\left \Vert \bx \right \Vert_\infty = \littleo(1)$.

\item The $\ell_\infty$ budget $\delta \leq  \left| b_i \right|/ \left\Vert \bw_i \right\Vert_1 - C \;\; \forall \, i = 1, \dots, M$ and for some $C = \Theta(1) > 0$. Note that this assumption is justified in an iterative/optimization-based attack, where the adversary gradually spends the budget over many iterations.

\item The number of neurons $M=\omega(1)$ as $N$ gets large. 

\item For each neuron $i=1, \dots, M$, we model the $\{w_i[k], \, k=1,\dots,N\}$ as i.i.d, with zero mean $\E\left[w_i[k] \right]=0$.
We assume that $\E\left[w_i[k]^2\right] = \sigma_i^2 = \Theta(1)$. 

\end{enumerate}
\end{assumptions}

\vspace{5pt}
\begin{theorem}\label{thm3}
With high probability, the high SNR condition ($\bsbar=\bs$) holds for $1-\littleo(1)$ fraction of neurons, i.e.
\begin{equation*}
\lim_{N \to \infty} \Pr \left(\frac{\left| S\right|}{M} = 1 - \littleo(1) \right) = 1,
\end{equation*}
where $S = \left\{ i: \left| {\bw^T_i \bx - b_i } \right| > \left|{\bw^T_i \be} \right| \right\}$.
\end{theorem}

\begin{proof}
We first state the following lemma (with proof in Appendix A):
\begin{lemma}\ \label{lemma:clt}
$\bw^T_i \bx \to \mathcal{N}(0,\sigma_i^2)$ in distribution.
\end{lemma}\vspace{3pt}


\noindent Noting that $\bw^T_i \bx - b_i \to \mathcal{N}(-b_i,\sigma_i^2)$, we can now write
\begin{align*}
\Pr &\left( \left| {\bw^T_i \bx - b_i }\right| > \left|{\bw^T_i \be} \right| \right)
\geq \Pr \left( \left| {\bw^T_i \bx - b_i}\right| > \delta \left\Vert\bw_i\right\Vert_1 \right) 
\\ &= \Pr \left( {\bw^T_i \bx - b_i } > \delta \left\Vert\bw_i\right\Vert_1 \right) + \Pr \left( {\bw^T_i \bx - b_i} < - \delta \left\Vert\bw_i\right\Vert_1 \right) \\
&= Q \left( \frac{\delta \left\Vert\bw_i\right\Vert_1 + b_i}{\sigma_i} \right)+Q \left( \frac{\delta \left\Vert\bw_i\right\Vert_1 - b_i}{\sigma_i} \right) \\
&\geq Q \left( \frac{\delta \left\Vert\bw_i\right\Vert_1 - \left|b_i\right|}{\sigma_i} \right)
=Q \left( \frac{\left\Vert\bw_i\right\Vert_1}{\sigma_i} \left(\delta - \frac{|b_i|}{\left\Vert \bw_i\right\Vert_1}\right) \right)\\
& = Q \left( \Theta(N) \, \left(\delta - \frac{|b_i|}{\left\Vert \bw_i\right\Vert_1}\right) \right) \to 1 \quad \text{as} \quad N\to \infty,
\end{align*}
where $Q(x) = \frac{1}{\sqrt{2\pi}} \int_x^\infty e^{-t^2/2} dt$ and  $\delta < \frac{|b_i|}{\left\Vert \bw_i\right\Vert_1}$ by Assumption 2. 
The theorem follows by using a union bound over $i=1,\dots,M$.
\end{proof}
 
\subsection{Attacks} \label{sec:attacks}

We assume that the adversary knows the true label, a reasonable (mildly pessimistic) assumption given the high accuracy of modern neural networks.
We consider attacks that focus on maximizing the first term in \eqref{eq:dist}, using a high SNR approximation to the distortion:
\begin{equation*} \label{eq:distortion}
\Delta = y(\bxhat) - y(\bx) = \bweq^T \proj\left( \be,\bx \right) = \be^T \proj\left( \bweq,\bx \right),
\end{equation*}
Here $\bweq \approx \bweq (\bx ) = \sum_{i=1}^M \, s_i \, w_0[i] \,\bw_i$ if we are applying a small perturbation to the input data.  However,
for iterative attacks with multiple small perturbations, $\bweq$ would evolve across iterations.

We can now define attacks in analogy with those for linear classifiers.  The adversary can use an ``effective input''
$\bx_1$ to compute the locally linear model $\bweq = \bweq (\bx_1 )$.  
For example, the adversary may choose $\bx_1 = \bx$ if making a small perturbation,
or may iterate computation of its perturbation using $\bx_1 = \bx + \be$.
The adversary can also use a possibly different ``effective input'' $\bx_2$ to estimate the set of basis coefficients retained by
the sparse front ends.  Armed with this notation, we can define two attacks:

\vspace{2pt} \textit{{Semi-white box:}}$\quad A_\mathrm{SW} ( \bx_1 , \epsilon )  = \epsilon \sign(\bweq (\bx_1 ))$,

\vspace{2pt} \textit{{White box:}}$\phantom{aa}\quad A_\mathrm{W} (\bx_1, \bx_2 , \epsilon) =  \epsilon \sign \left( \proj\left( \bweq (\bx_1 ) ,\bx_2 \right) \right)$. \vspace{2pt}

\noindent We make no claims on the optimality of these attacks. They are simply sensible strategies based on the locally linear model, and as we show in the next section, they are more powerful than existing FGSM attacks for multiclass classification. 

For simplicity, we set $\bx_1 = \bx_2$ for the white box attack, and simplify notation by denoting it by $A_\mathrm{W} (\bx_1 , \epsilon)$.  
A (suboptimal) default choice is to set $\bx_1 = \bx_2 =  \bx$, relying on a high SNR approximation for both the network switches 
and for the sparsifying front end.  However, we can also refine these choices iteratively, as follows.

\vspace{5pt}
\noindent {\bf Iterative versions:} We choose a particularly simple approach, in which we use a small attack budget $\delta$ to change $\bweq$ by small amounts
and update the ``direction'' of the attack, maintaining the overall $\ell_{\infty}$ constraint at each stage:
\begin{gather*}
\be_{k+1}  = \be_k + A(\bx + \be_k, \delta) \\
\be_{k+1}  = \mathrm{clip}_{{\epsilon}}(\be_{k+1}),
\end{gather*}
where $\mathrm{clip}_{{\epsilon}}(\be) \triangleq \max(\min(\be,\epsilon),-\epsilon)$. We believe there is room for improvement in how we iterate, but this particular choice suffices to illustrate the power of locally linear modeling. 

\begin{remark}
The Fast Gradient Sign Method (FGSM) attack puts its attack budget along the gradient of the cost
function $J (\cdot )$ used to train the network.
For binary classification and the cross-entropy cost function, we can show that it is equivalent to the semi-white box attack with $\bx_1 = \bx$.
Specifically, we can show that 
\begin{align*}
\be_\mathrm{FGSM} &= \epsilon \sign\left(\nabla_\bx J(\bx, l) \right) = A_\mathrm{SW} ( \bx , \epsilon ) 
\end{align*}
where $l$ is the true label, by verifying that the gradient is proportional to $\bweq ( \bx )$.
For a larger number of classes, however, insights from our locally linear modeling can be used to devise more powerful
attacks than FGSM.

\end{remark}

\subsection{Multiclass Classification} \label{sec:multiclass}

\begin{figure}[t]
  \centering
  	\begin{tikzpicture}[
	rect_h/.style = {rectangle, draw=black!100, fill=cyan!18, thin, minimum height=7.5mm, minimum width=11mm},
	rect_v/.style = {rectangle, draw=black!100, fill=cyan!18, thin, minimum height=12mm, minimum width=7mm},
	rect_v_big/.style = {rectangle, draw=black!100, fill=cyan!18, thin, minimum height=62mm, minimum width=7mm},
	rect_v_big_2/.style = {rectangle, draw=black!100, fill=cyan!18, thin, minimum height=62mm, minimum width=9mm},
	circ/.style = {circle, draw=black!100, fill=cyan!18, thin, minimum size=7mm},
	outer/.style = {rounded corners=0.2cm, draw=black!100, dotted, thick, inner sep = 2.5mm},
	rect/.style = {rectangle, draw=black!100, fill=cyan!18, thin, minimum height=7mm, minimum width=7mm},
	rect_smaller/.style = {rectangle, draw=black!100, fill=cyan!18, thin, minimum height=8mm, minimum width=8mm},
	scale=0.95, every node/.style={scale=0.95}
	]
	\node[] (x) {$\bx$};
	\node[] (conn) [right = 6mm of x] {};
	\node[rect_smaller] (w2) [above right = 8mm of conn, align = center] {$\bw^{\{2\}}_{\mathrm{eq}}$ \\[5pt]$b^{\{2\}}_{\mathrm{eq}}$};	
	\node[rect_smaller] (w1) [above = 5mm of w2, align = center] {$\bw^{\{1\}}_{\mathrm{eq}}$\\[5pt] $b^{\{1\}}_{\mathrm{eq}}$};
	\node[rect_smaller] (wm) [below = 17mm of w2, align = center] {$\bw^{\{L\}}_{\mathrm{eq}}$\\[5pt] $b^{\{L\}}_{\mathrm{eq}}$};
	\node[] (dots) [above = 5mm of wm] {$\vdots$};
	\node[] (dots_hidden) [above = 3.8mm of dots] {};
	\node[rect_v_big_2] (softmax) [right = 15mm of dots_hidden, align = center] {Soft\\max};
	\node[] (w1_hidden) [right = 11.4mm of w1] {};
	\node[] (w2_hidden) [right = 11.4mm of w2] {};
	\node[] (wm_hidden) [right = 11mm of wm] {};
	\node[] (y1_hidden) [right = 6.5mm of w1_hidden] {};
	\node[] (y2_hidden) [right = 6.5mm of w2_hidden] {};
	\node[] (ym_hidden) [right = 6.5mm of wm_hidden] {};
	\node[] (out1) [right = 3mm of y1_hidden] {$p_1$};
	\node[] (out2) [right = 3mm of y2_hidden] {$p_2$};
	\node[] (outm) [right = 3mm of ym_hidden] {$p_{\scriptscriptstyle{L}}$};
	\begin{scope}[decoration={markings,
    mark=at position 0.5 with {\arrow{Latex[width=1.2mm,length=1.8mm]}}}] 
    \draw[postaction={decorate}] (x.north east) -- (w1.west);
    \draw[postaction={decorate}] (x.east) -- (w2.west); 
    \draw[postaction={decorate}] (x.south east) -- (wm.west);
    \draw[postaction={decorate}] (w1.east) -- node[anchor=240]{$y_1$} (w1_hidden.west);
    \draw[postaction={decorate}] (w2.east) -- node[anchor=240]{$y_2$} (w2_hidden.west); 
    \draw[postaction={decorate}] (wm.east) -- node[anchor=240]{$y_{\scriptscriptstyle{L}}$} (wm_hidden.west);    
	\end{scope}
	\begin{scope}[decoration={markings,
    mark=at position 0.7 with {\arrow{Latex[width=1.2mm,length=1.8mm]}}}] 
	\end{scope}
	\draw[->] (y1_hidden.west) -- node[anchor=240]{} (out1.west); 
    \draw[->] (y2_hidden.west) -- node[anchor=240]{} (out2.west); 
    \draw[->] (ym_hidden.west) -- node[anchor=240]{} (outm.west); 
	\node[outer, fit = (w1) (wm), label=below:\small{\begin{tabular}{l}Equivalent hidden layer\end{tabular}}] (n0) {};
	\end{tikzpicture}
	\captionof{figure}{Multilayer (deep) neural network with $L$ classes. Each of the $L$ pre-softmax outputs (logits) is locally linear: $y_i = \bweqi^T\,\bx - \beqi, \;i=1,\dots,L$.}
	\label{fig:general-net}
\end{figure}
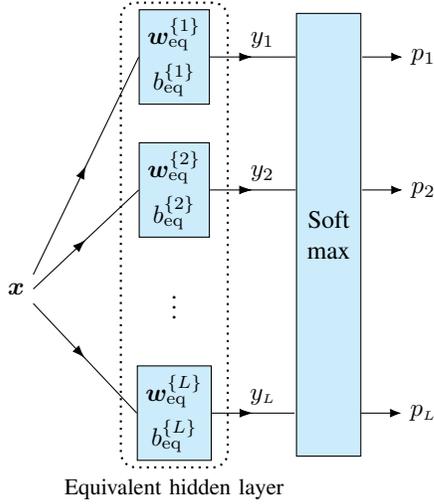

In this subsection, we consider a multilayer (deep) network with $L$ classes. Each of the outputs of the network can be modeled using the analysis in the previous section as follows:
\begin{equation*}
y_i = \bweqi^T\bx -\beqi, \; \; i=1,\dots,L,
\end{equation*}
where $\by=[y_1,y_2,..., y_L]^T$, $p_i=S_i(\by)$, and  the softmax function
$S_i(\by)={e^{y_i}}/\left({\sum_{j=1}^L e^{y_j}}\right).
$
Assume that $\bx$ belongs to class $t$ (with label $t$ known to the adversary). 
\vspace{5pt}

\noindent \textbf{Locally linear attack:} The adversary can sidestep the nonlinearity of the softmax layer, since its goal is simply to make $y_i > y_t$ for {\it some} $i \neq t$.  Thus, the adversary can consider
$L-1$ binary classification problems, and solve for perturbations aiming to maximize $y_i - y_t$ for each $i \neq t$.  
We now apply the semi-white and white box attacks, and their iterative versions, to each pair, with $\bweq = \bweqi - \bweqt$
being the equivalent locally linear model from the input to $y_i - y_t$. After computing the distortions for
each pair, the adversary applies its attack budget to the {\it worst-case} pair for which the distortion is the largest:
\vspace{-1pt}
\begin{equation*}
\max_{i,\be} \quad y_i(\bx+\be)-y_t(\bx+\be), \qquad \text{s.t.} \quad \left\Vert \be \right\Vert_{\infty} \leq \epsilon
\end{equation*}

\noindent \textbf{FGSM:} Unfortunately, the FGSM attack does not have a clean interpretation in the multiclass setting.  Taking the gradient of
the cross-entropy betwen one-hot encoded vector of the true label $\bl$ ($l[k]=\delta_{tk}$) and the final output of the model $\bp=[p_1,p_2,...,p_L]$, with
$J(\bl,\bp)=- \sum_{i=1}^{L} l_i \log{(p_i)} 
$, we obtain 
\vspace{-5pt}
\begin{equation*}
\be_{\mathrm{FGSM}} = \epsilon \sign \bigg( \bweqt (p_t-1)
+ \sum_{
k \neq t} \bweqk p_k \bigg).
\end{equation*}
This does not take the most direct approach to corrupting the desired label, unlike
our locally linear attack, and is expected to perform worse.


\begin{figure}[!t]
\centering
\subfigure[A natural image and its adversarial counterpart (misclassified as ``7'').]{\label{fig:image_adv}\includegraphics[width=0.95\columnwidth]{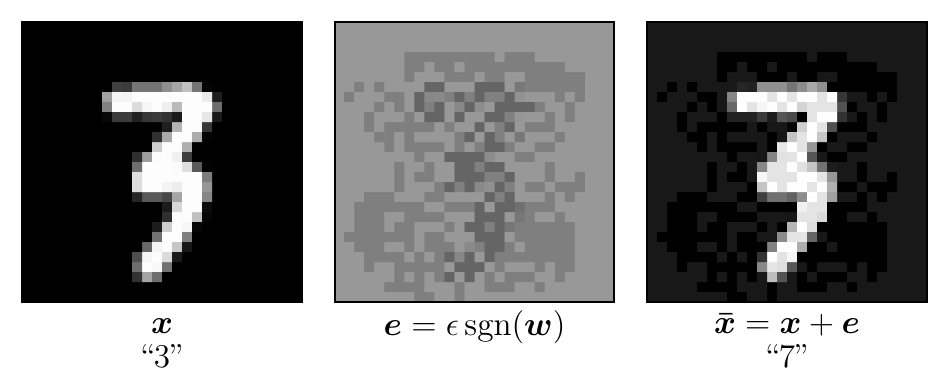}}
\subfigure[After sparsification, the perturbed image is no longer adversarial.]{\label{fig:image_adv_sp}\includegraphics[width=0.95\columnwidth]{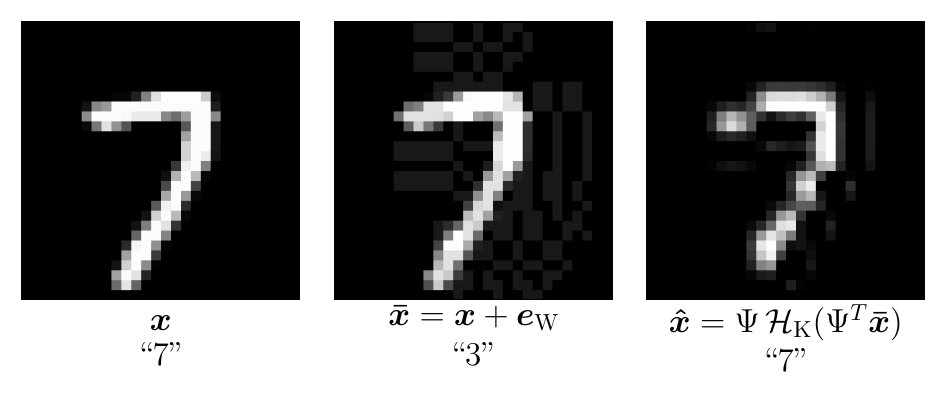}}
\caption{Sample images depicting the interplay between attack and defense: tiny adversarial attacks can fool a classifier, but sparsity-based preprocessing can restore accuracy by projecting the attack down to a lower dimensional subspace.}
\vskip -0.5em
\end{figure}

\begin{table}[!b]
\centering
\caption{Binary classification accuracies for linear SVM, with $\epsilon=0.1$ for attacks and $\rho = 2\%$ for defense.}
\label{table:binary}
\begin{center}
\begin{small}
\begin{tabular}{@{}lcc@{}}
\toprule
 & No defense & 
\multicolumn{1}{c}{\begin{tabular}[l]{@{}c@{}}Sparsifying \\ front end\end{tabular}}\\ \midrule
Semi-white box attack & 0.25 & 98.37 \\
White box attack & 0.25 & 95.37 \\
\bottomrule                    
\end{tabular}
\end{small}
\end{center}
\vskip -0.5em
\end{table}

\begin{figure}[!t]
\centering
\subfigure[Accuracy vs.\ sparsity level, where the attacks use $\epsilon=0.1$.]{\label{fig:acc_vs_sp}\includegraphics[width=0.99\columnwidth]{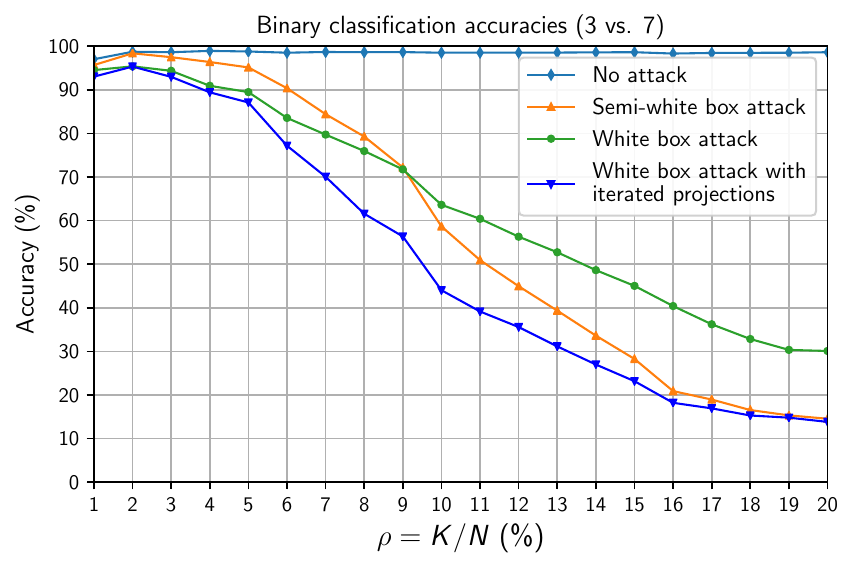}}
\par\smallskip
\subfigure[Accuracy vs.\ attack budget, where the defense uses $\rho\,$=$\,$2\%.]{\label{fig:acc_vs_eps}\includegraphics[width=0.99\columnwidth]{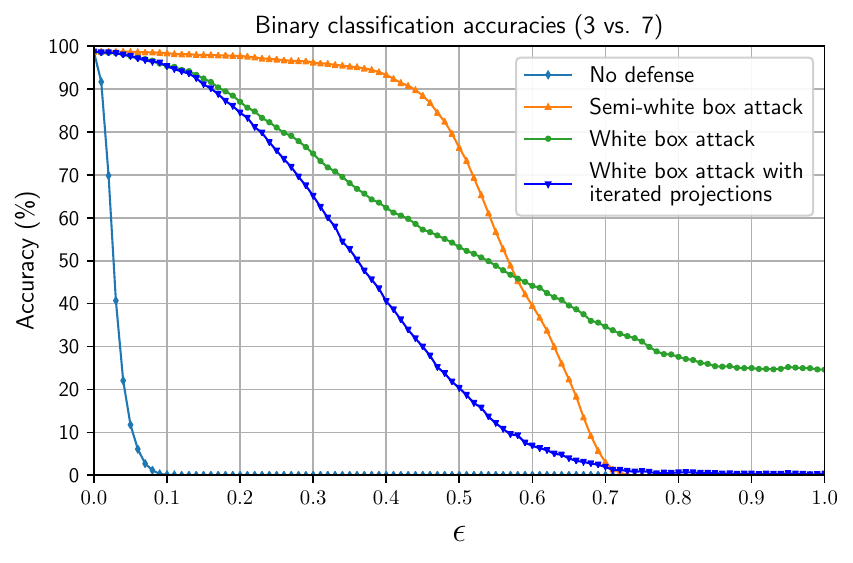}}
\caption{Binary classification accuracies for the linear SVM as a function of the sparsity level $\rho$ and attack budget $\epsilon$.}
\label{fig_binary_svm}
\vskip -0.5em
\end{figure}

\section{Experimental Results} \label{sec:exp}

In this section we demonstrate the effectiveness of sparsifying front ends on inference tasks on the MNIST \cite{lecun1998gradient} and CIFAR-10 \cite{krizhevsky2009learning} datasets. Code is available at \href{https://github.com/soorya19/sparsity-based-defenses/}{https://github.com/soorya19/sparsity-based-defenses/}.
We first provide results for binary linear classifiers (``3'' versus ``7''), and then report on experiments with neural networks, for both binary
and multi-class classification.  

\subsection{Linear Classifiers} \label{exp_linear}

\vspace{2pt}
\noindent\textbf{Set-up:} Here we train a linear SVM $g(\bx) = \bw^T\bx \,+\, b$ to classify digit pairs $d_1$ and $d_2$ from the MNIST dataset. The ``direction'' of the attack is opposite that of the correct class: if the SVM predicts class $d_1$ when $g(\bx) < 0$ and $d_2$ when $g(\bx) > 0$, the perturbation is of the form $\bxbar = \bx + \epsilon \sign \left(\bw\right)$ for images of class $d_1$, and $\bxbar = \bx -\epsilon \sign \left(\bw\right)$ for class $d_2$. We assume that the adversary has access to the true labels.
For the defense, we use the Daubechies-5 wavelet \cite{daubechies1992ten} to perform sparsification and retrain the SVM with sparsified images (for various sparsity levels) before evaluating performance. 

\vspace{5pt}
\noindent\textbf{Results:} We consider the case of 3 versus 7 classification. When no defense is present, an attack\footnote{The reported values of $\epsilon$ correspond to images normalized to $[0,1].$} with $\epsilon = 0.1$ renders the classifier useless, with accuracy dropping from 98.64\% to 0.25\% as depicted in Fig. \ref{fig:image_adv}.

Fig. \ref{fig:image_adv_sp} illustrates the sparsifying front end at work, showing an image, its perturbed version and the effect of sparsification. Insertion of the front end greatly improves resilience to adversarial attacks: as shown in Table \ref{table:binary}, accuracy is restored to near-baseline levels at low sparsity levels. As discussed in Section \ref{sec:high_snr}, the choice of sparsity level $\rho$ must optimize the tradeoff between attack attenuation and unwanted signal distortion. We find that a value of $\rho$ between 1--5\% works well for all digit pairs, with $\rho = 2$\% being the optimal choice for the 3 vs.\ 7 scenario. 

Fig. \ref{fig_binary_svm} reports on accuracy as a function of sparsity level $\rho$ and attack budget $\epsilon$.  
At the small values of $\rho$ and $\epsilon$ that we are interested in, the white box attack causes more damage than the semi-white box attack. At larger $\rho$ and $\epsilon$, it performs poorer than the semi-white box attack: the high SNR condition in Proposition \ref{prop:SNR_condition} is no longer satisfied, hence the white box attack is attacking the ``wrong subspace''. 
It is easy to devise iterative white box attacks that do better by refining the estimate of the $K$-dimensional subspace in the following manner:
\begin{equation*}
{\be}^{[i+1]} = {\be}^{[i]} + \delta \sign \left(\proj \left({\bw},\,\bx + {\be}^{[i]} \right)\right),
\end{equation*}
where $\delta \leq \epsilon$ and ${\be}^{[0]} = 0$. 
Essentially, we refine our estimate of $\support(\bx+\be)$ at each iteration by calculating the projection of $\bw$ onto the top $K$ basis vectors of $\bx + \be$ (rather than just $\bx$). We can observe from the figures that the attack with iterated projections performs better in the low SNR region.
However, this scenario is not of practical interest, since front ends with large $\rho$ do not provide enough attenuation of the adversarial perturbation, and perturbations with large $\epsilon$ are no longer visually imperceptible.

\subsection{Neural Networks} \label{exp_nn}

\begin{table}[!t]
\vspace{7pt}
\centering
\caption{Multiclass MNIST classification accuracies for 4-layer CNN, with $\epsilon=0.2$ for attacks and $\rho = 3.5\%$ for defense.}
\vskip -0.5em
\label{table_multiclass-mnist}
\begin{center}
\begin{small}
\begin{tabular}{lcc}
\toprule & 
\multicolumn{1}{c}{\begin{tabular}[l]{@{}c@{}}No defense\end{tabular}} & 
\multicolumn{1}{c}{\begin{tabular}[l]{@{}c@{}}Sparsifying \\ front end\end{tabular}} \\
\midrule
No attack & 99.31 & 98.97 \\
Iterative locally linear attack & 7.36 & 74.38
\\
Iterative FGSM & 6.34 & 74.97
 \\ 
Momentum iterative FGSM & 6.99 & 73.55
 \\
PGD (100 random restarts) & 5.12 & 61.04 \\
\bottomrule
\end{tabular}
\end{small}
\end{center}
\vskip -1.2em
\end{table}

\begin{figure}[!t]
\centering
\includegraphics[width=0.99\columnwidth]{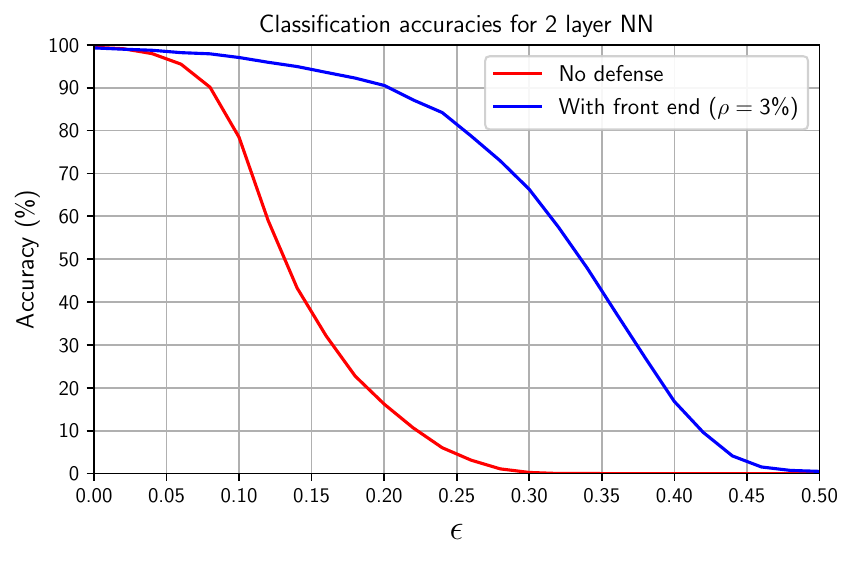}
\caption{Binary classification accuracies for 2-layer NN as a function of $\epsilon$, using a 1000-step iterative locally linear attack.}
\label{fig_binary_nn}
\vskip -1em
\end{figure}

For neural networks, we perturb images with the following attacks in the white box setting: 
\begin{enumerate*}[(a)]
\item the locally linear attack,
\item its iterative version,
\item FGSM \citep{goodfellow2014adversarial},
\item iterative FGSM \citep{kurakin2016scale},
\item projected gradient descent (PGD) \citep{madry2017towards},
and
\item momentum iterative FGSM \citep{dong2017discovering} (winner of the NIPS 2017 competition on adversarial attacks and defenses \citep{kurakin2018adversarial}).
\end{enumerate*}
For PGD, we use multiple random restarts and calculate accuracy over the most successful restart(s) for each image.
We evaluate three versions of the attacks: one that uses the backward pass differential approximation (BPDA) technique of \cite{obfuscated-gradients} to approximate the gradient of the front end as $\bm{1}$, a second version where the gradient is calculated as the projection onto the top $K$ basis vectors of the input, and a third version where we iteratively refine the projection as described in the previous section. We report accuracies for the version that causes the most damage.

\vspace{5pt}
\noindent\textbf{Set-up:} For binary classification of MNIST digits ``3'' and ``7'', we use a 2-layer fully-connected network with 10 neurons. For multi-class MNIST, we use a 4-layer CNN consisting of two convolutional layers (with 20 and 40 feature maps, each with 5x5 filters) and two fully-connected layers (containing 1000 neurons each, with dropout) \citep{deeplearning}. 
For CIFAR-10, we use a 32-layer ResNet \cite{he2016deep} and follow the data augmentation strategy of \cite{he2016deep} for training.
For the sparsifying front end, we use the Daubechies-5 wavelet for binary MNIST, Coiflet-1 for multiclass MNIST and Symlet-9 for CIFAR-10,
and retrain the networks with sparsified images. 

\vspace{5pt}
\noindent\textbf{MNIST results:} Figure \ref{fig_binary_nn} reports on binary classification accuracies for the 2-layer NN as a function of attack budget, showing that the front end improves robustness across a range of $\epsilon$.
As shown in Section \ref{sec:attacks}, FGSM is identical to the locally linear attack for binary classification, so we do not label it in the figure. We note that images are normalized to the range $[0,1]$, so by the end of the chosen range of $\epsilon$, perturbations are no longer visually imperceptible. 

Table \ref{table_multiclass-mnist} reports on multiclass classification accuracies for the 4-layer CNN, with attacks using an $\ell_\infty$ budget of $\epsilon=0.2$. Iterative attacks are run for 1000 steps with a per-iteration budget of $\delta=0.01$, except for PGD and the iterative locally linear attack which use 100 steps of $\delta=0.05$.
Without any defense, a strong adversary can significantly degrade performance, reducing accuracy from 99.31\% to 5.12\%. In contrast, when a sparsifying front end is present (with sparsity level $\rho=3.5\%$), the network robustness measurably improves, increasing accuracy to 61.04\% in the worst-case scenario. We note that the locally linear attack is stronger than FGSM (with results in Appendix B), and the iterative locally linear attack is competitive with single runs of the other iterative attacks.
\begin{table}[!t]
\centering
\begin{center}
\begin{small}
\caption{\small Comparison to other defenses on MNIST. For defenses with certified bounds, numbers in blue denote lower bounds on adversarial accuracy. Numbers in black denote PGD attack accuracy with 100 steps and 100 random restarts. }
\label{table:mnist_comparison}
\begin{tabularx}{0.98\columnwidth}{@{\hspace{6pt}}cl@{\hspace{20pt}}l@{\hspace{10pt}}c} 
\toprule
\begin{tabular}{@{}c}Attack \\budget 
\end{tabular}& Defense & 
\begin{tabular}{@{\hspace{-3pt}}l}
Adversarial \\ accuracy \end{tabular} &
\begin{tabular}{l}
Natural \\ accuracy\end{tabular} \\ 
\midrule&&\\[-8pt]
$\epsilon = 0.1$ 
& Sparsifying front end & 92.18 & 98.97\\
& Madry et al & 95.83 & 99.54 \\
& {Raghunathan et al} & 88.00, {\color{DarkBlue} 65.00} & 95.72\\
& {Wong et al} &  97.33, {\color{DarkBlue} 96.33}  & 98.81 \\[1.5pt]
\midrule&&\\[-8pt]
$\epsilon = 0.2$ 
& Sparsifying front end & 61.04 & 98.97\\
& Madry et al & 92.95 & 99.20\\[1.5pt]
 \midrule&&\\[-8pt]
 $\epsilon = 0.3$ 
 & Sparsifying front end & 13.15 &  98.97\\
 & Madry et al & 90.57 & 99.15 \\
 & {Wong et al}  &  86.21, {\color{DarkBlue} 54.34}  & 88.84 \\[1.5pt]
\bottomrule
\end{tabularx}
\end{small}
\end{center}
\vskip -1.5em
\end{table}

Tables \ref{table:mnist_comparison} compares the front end defense with the state-of-the-art empirical defense of \citet{madry2017towards} and the provable defenses of \citet{raghunathan2018certified} and \citet{wong2018scaling}. We run the adversarial training defense on our network architecture and report on accuracies obtained. For \cite{raghunathan2018certified} and \cite{wong2018scaling}, we use publicly available pretrained models and report both empirical and certified performance. 
When the attack budget is small ($\epsilon=0.1$), the sparsifying front end is competitive with existing defenses, with better accuracies than the provable defense of \cite{raghunathan2018certified}. At the higher budget of $\epsilon=0.2$, the adversarial training defense of Madry et al provides better performance. The front end still offers a measurable degree of robustness, with the advantage that it is computationally much less expensive than adversarial training, which requires one to retrain the network with adversarial images, resulting in an $M$-fold slowdown if we use $M$ steps for the attack \cite{madry2017towards}.
For $\epsilon=0.3$, the performance of the sparsifying front end drops to 13.15\%.  This is not entirely surprising, since our theoretical guarantees
apply for a high SNR regime corresponding to relatively small perturbations.  We believe this is because sparse projections alone, especially with an off-the-shelf basis, are not ``nonlinear enough'' to discriminate between desired signal and perturbation for large attack budgets.

We have also performed experiments that combine the front end with adversarial training, with the result that accuracy improves from 61.04\% (when using the front end alone) to 80.07\% at $\epsilon=0.2$. We note that there is no performance increase compared to pure adversarial training if the same value of $\epsilon$ is used for both training and testing. However, sparsification seems to provide some robustness to mismatch between training and test conditions;
in particular,  if we test at a value of $\epsilon$ larger than what the network was trained for.
For example, accuracy increases from 35.65\% (using adversarial training alone) to 69.04\% when we train at $\epsilon=0.1$ and test at $\epsilon=0.2$.  These results are obtained using a 100-step PGD attack with 100 random restarts; in the training phase, we use a 40-step PGD attack as in \cite{madry2017towards}.

\vspace{5pt}
\noindent\textbf{CIFAR-10 results:} 
Table \ref{table:cifar} reports on CIFAR-10 accuracies with a 32-layer ResNet. Front end sparsification is performed after converting color images to the ($Y$, $C_b$, $C_r$) decorrelated color space and then projecting to the wavelet basis. Since human vision is relatively less sensitive to chrominance, we impose higher sparsity in the $C_b$ and $C_r$ axes. 

When using the front end alone, we can improve robustness from 11.32\% to 39.60\% at sparsity levels of (3.5, 1, 1)\%. However, there is a drop in natural accuracy due to the high level of sparsity imposed. We can significantly improve both metrics by combining the front end defense, using less drastic sparsity levels, with adversarial training.
Specifically, a front end with sparsity levels (50, 25, 25)\%, together with adversarial training using 5 steps of PGD, results in 67.21\% adversarial and 88.08\% natural accuracies, respectively. 
We note that the accuracy under attack is 0.4\% better than using adversarial training alone.

\begin{table}[!t]
\centering
\begin{center}
\begin{small}
\caption{\small CIFAR-10 accuracies for ResNet-32 at $\epsilon=2/255$. Defenses are tested with a 1000-step PGD attack.}
\label{table:cifar}
\begin{tabular}{lcc} 
\toprule
Defense & 
\begin{tabular}{l}
Adversarial \\ accuracy \end{tabular} &
\begin{tabular}{l}
Natural \\ accuracy\end{tabular} \\ 
\midrule
No defense & 11.32 & 91.11 \\
Sparsifying front end & 39.60 & 62.30\\
Adversarial training & 66.82 & 88.63 \\
Front end + adv. training & 67.21 & 88.02\\
\bottomrule
\end{tabular}
\end{small}
\end{center}
\vskip -1.5em
\end{table}


\section{Conclusions}

Our results make the case that sparsity is a crucial tool for limiting the impact of adversarial attacks on neural networks. We have also shown that a ``locally linear'' model for the network, an implicit premise behind state-of-the-art iterative attacks,  provides key design insights, both for devising and combating adversarial perturbations. Our proposed sparsifying front end makes an implicit assumption on the generative model for the data that we believe must hold quite generally for high-dimensional data, in order to evade the curse of dimensionality. 
We believe that these results are the first steps towards establishing a comprehensive design framework, firmly grounded in theoretical fundamentals, for robust neural networks, that is complementary to alternative defenses based on modifying the manner in which networks are trained.

While many state-of-the-art defenses are based on modifying the optimization involved in training the overall neural network, ours is
a bottom-up approach to robustness which is potentially more amenable to interpretation, and to theoretical guarantees based on a 
statistical characterization of the input.  However, much further work is required in order to realize this potential.
First, developing sparse generative models matched to various datasets is required for design of sparsifying front ends.  Even for the simple MNIST dataset considered here, 
the orthogonal wavelet basis considered here is only a first guess, and we believe that it can be improved upon by learning from data, and by use of overcomplete bases.  
Second, our placeholder scheme of picking the largest $K$ coefficients could be improved by devising computationally efficient and data-adaptive techniques for enforcing sparsity.  Lastly, while our work highlights the role of sparse projections in attenuating perturbations,
our theoretical framework applies to a high SNR regime corresponding to relatively small attack budgets. In practice,
the accuracy with our defense deteriorates for large attack budgets.  Thus, sparse projections alone are not enough to provide robustness 
against large adversarial perturbations, and additional ideas are needed to construct a bottom-up approach that is competitive with the
current state-of-the-art defense based on adversarial training of the entire network.  

While we have restricted attention to simple datasets and small networks in this paper in order to develop insight, 
the design of larger (deeper) networks for more complex datasets is our ultimate objective.  Our preliminary results for
CIFAR-10 with a 32-layer ResNet are encouraging, showing that some level of sparsification, along with adversarial training,
yields slightly better accuracy under attack than adversarial training alone. However, detailed insight into the front end and
network {\it structure} required to attain robustness remains a wide open research problem.


\section*{Acknowledgments}

This work was supported in part by the National Science Foundation under grants CNS-1518812, CCF-1755808 and CCF-1909320, by Systems on Nanoscale Information fabriCs (SONIC), one of the six SRC STARnet Centers, sponsored by MARCO and DARPA, and by the UC Office of the President under grant No. LFR-18-548175.

{\footnotesize
\bibliographystyle{IEEEtranN}
\bibliography{main.bbl}

\begin{thebibliography}{53}
\providecommand{\natexlab}[1]{#1}
\providecommand{\url}[1]{#1}
\csname url@samestyle\endcsname
\providecommand{\newblock}{\relax}
\providecommand{\bibinfo}[2]{#2}
\providecommand{\BIBentrySTDinterwordspacing}{\spaceskip=0pt\relax}
\providecommand{\BIBentryALTinterwordstretchfactor}{4}
\providecommand{\BIBentryALTinterwordspacing}{\spaceskip=\fontdimen2\font plus
\BIBentryALTinterwordstretchfactor\fontdimen3\font minus
  \fontdimen4\font\relax}
\providecommand{\BIBforeignlanguage}[2]{{%
\expandafter\ifx\csname l@#1\endcsname\relax
\typeout{** WARNING: IEEEtranN.bst: No hyphenation pattern has been}%
\typeout{** loaded for the language `#1'. Using the pattern for}%
\typeout{** the default language instead.}%
\else
\language=\csname l@#1\endcsname
\fi
#2}}
\providecommand{\BIBdecl}{\relax}
\BIBdecl

\bibitem[Szegedy et~al.(2014)Szegedy, Zaremba, Sutskever, Bruna, Erhan,
  Goodfellow, and Fergus]{szegedy2013intriguing}
C.~Szegedy, W.~Zaremba, I.~Sutskever, J.~Bruna, D.~Erhan, I.~Goodfellow, and
  R.~Fergus, ``Intriguing properties of neural networks,'' in
  \emph{International Conference on Learning Representations}, 2014.

\bibitem[Goodfellow et~al.(2015)Goodfellow, Shlens, and
  Szegedy]{goodfellow2014adversarial}
I.~J. Goodfellow, J.~Shlens, and C.~Szegedy, ``Explaining and harnessing
  adversarial examples,'' in \emph{International Conference on Learning
  Representations}, 2015.

\bibitem[Fawzi et~al.(2017{\natexlab{a}})Fawzi, Moosavi-Dezfooli, and
  Frossard]{fawzi2017review}
A.~Fawzi, S.-M. Moosavi-Dezfooli, and P.~Frossard, ``The robustness of deep
  networks: A geometrical perspective,'' \emph{IEEE Signal Processing
  Magazine}, vol.~34, no.~6, pp. 50--62, 2017.

\bibitem[Madry et~al.(2018)Madry, Makelov, Schmidt, Tsipras, and
  Vladu]{madry2017towards}
A.~Madry, A.~Makelov, L.~Schmidt, D.~Tsipras, and A.~Vladu, ``Towards deep
  learning models resistant to adversarial attacks,'' in \emph{International
  Conference on Learning Representations}, 2018.

\bibitem[Qin et~al.(2019)Qin, Martens, Gowal, Krishnan, Dvijotham, Fawzi, De,
  Stanforth, and Kohli]{qin2019adversarial}
C.~Qin, J.~Martens, S.~Gowal, D.~Krishnan, K.~Dvijotham, A.~Fawzi, S.~De,
  R.~Stanforth, and P.~Kohli, ``Adversarial robustness through local
  linearization,'' in \emph{Advances in Neural Information Processing Systems},
  2019, pp. 13\,824--13\,833.

\bibitem[LeCun et~al.(1998)LeCun, Bottou, Bengio, and
  Haffner]{lecun1998gradient}
Y.~LeCun, L.~Bottou, Y.~Bengio, and P.~Haffner, ``Gradient-based learning
  applied to document recognition,'' \emph{Proceedings of the IEEE}, vol.~86,
  no.~11, pp. 2278--2324, 1998.

\bibitem[Krizhevsky and Hinton(2009)]{krizhevsky2009learning}
A.~Krizhevsky and G.~Hinton, ``Learning multiple layers of features from tiny
  images,'' University of Toronto, Tech. Rep., 2009.

\bibitem[Mallat and Zhang(1993)]{mallat1993matching}
S.~Mallat and Z.~Zhang, ``Matching pursuits with time-frequency dictionaries,''
  \emph{IEEE Transactions on Signal Processing}, vol.~41, no.~12, pp.
  3397--3415, 1993.

\bibitem[Blumensath and Davies(2008)]{blumensath2008gradient}
T.~Blumensath and M.~E. Davies, ``Gradient pursuits,'' \emph{IEEE Transactions
  on Signal Processing}, vol.~56, no.~6, pp. 2370--2382, 2008.

\bibitem[Bruckstein et~al.(2009)Bruckstein, Donoho, and
  Elad]{bruckstein2009sparse}
A.~M. Bruckstein, D.~L. Donoho, and M.~Elad, ``From sparse solutions of systems
  of equations to sparse modeling of signals and images,'' \emph{SIAM Review},
  vol.~51, no.~1, pp. 34--81, 2009.

\bibitem[Donoho et~al.(1994)Donoho, Johnstone, et~al.]{donoho1994ideal}
D.~L. Donoho, I.~M. Johnstone \emph{et~al.}, ``Ideal denoising in an
  orthonormal basis chosen from a library of bases,'' \emph{Comptes rendus de
  l'Acad{\'e}mie des sciences. S{\'e}rie I, Math{\'e}matique}, vol. 319,
  no.~12, pp. 1317--1322, 1994.

\bibitem[Aharon et~al.(2006)Aharon, Elad, and Bruckstein]{aharon2006rm}
M.~Aharon, M.~Elad, and A.~Bruckstein, ``$k$-{SVD}: {An} algorithm for
  designing overcomplete dictionaries for sparse representation,'' \emph{IEEE
  Transactions on Signal Processing}, vol.~54, no.~11, pp. 4311--4322, 2006.

\bibitem[Gorodnitsky and Rao(1997)]{gorodnitsky1997sparse}
I.~F. Gorodnitsky and B.~D. Rao, ``Sparse signal reconstruction from limited
  data using {FOCUSS}: {A} re-weighted minimum norm algorithm,'' \emph{IEEE
  Transactions on Signal Processing}, vol.~45, no.~3, pp. 600--616, 1997.

\bibitem[Wright et~al.(2009)Wright, Nowak, and Figueiredo]{wright2009sparse}
S.~J. Wright, R.~D. Nowak, and M.~A. Figueiredo, ``Sparse reconstruction by
  separable approximation,'' \emph{IEEE Transactions on Signal Processing},
  vol.~57, no.~7, pp. 2479--2493, 2009.

\bibitem[Yang et~al.(2010)Yang, Wright, Huang, and Ma]{yang2010image}
J.~Yang, J.~Wright, T.~S. Huang, and Y.~Ma, ``Image super-resolution via sparse
  representation,'' \emph{IEEE Transactions on Image Processing}, vol.~19,
  no.~11, pp. 2861--2873, 2010.

\bibitem[Elad et~al.(2005)Elad, Starck, Querre, and
  Donoho]{elad2005simultaneous}
M.~Elad, J.-L. Starck, P.~Querre, and D.~L. Donoho, ``Simultaneous cartoon and
  texture image inpainting using morphological component analysis ({MCA}),''
  \emph{Applied and Computational Harmonic Analysis}, vol.~19, no.~3, pp.
  340--358, 2005.

\bibitem[Zibulevsky and Pearlmutter(2001)]{zibulevsky2001blind}
M.~Zibulevsky and B.~A. Pearlmutter, ``Blind source separation by sparse
  decomposition in a signal dictionary,'' \emph{Neural Computation}, vol.~13,
  no.~4, pp. 863--882, 2001.

\bibitem[Malioutov et~al.(2005)Malioutov, Cetin, and
  Willsky]{malioutov2005sparse}
D.~Malioutov, M.~Cetin, and A.~S. Willsky, ``A sparse signal reconstruction
  perspective for source localization with sensor arrays,'' \emph{IEEE
  Transactions on Signal Processing}, vol.~53, no.~8, pp. 3010--3022, 2005.

\bibitem[Mallat(1999)]{mallat1999wavelet}
S.~Mallat, \emph{A Wavelet Tour of Signal Processing}.\hskip 1em plus 0.5em
  minus 0.4em\relax Elsevier, 1999.

\bibitem[Skretting and Engan(2010)]{skretting2010recursive}
K.~Skretting and K.~Engan, ``Recursive least squares dictionary learning
  algorithm,'' \emph{IEEE Transactions on Signal Processing}, vol.~58, no.~4,
  pp. 2121--2130, 2010.

\bibitem[Rubinstein et~al.(2013)Rubinstein, Peleg, and
  Elad]{rubinstein2013analysis}
R.~Rubinstein, T.~Peleg, and M.~Elad, ``Analysis {K-SVD}: {A}
  dictionary-learning algorithm for the analysis sparse model,'' \emph{IEEE
  Transactions on Signal Processing}, vol.~61, no.~3, pp. 661--677, 2013.

\bibitem[Makhzani and Frey(2014)]{makhzani2013ksparse}
A.~Makhzani and B.~Frey, ``$k$-{Sparse} autoencoders,'' in \emph{International
  Conference on Learning Representations}, 2014.

\bibitem[Moosavi-Dezfooli et~al.(2016)Moosavi-Dezfooli, Fawzi, and
  Frossard]{moosavi2016deepfool}
S.-M. Moosavi-Dezfooli, A.~Fawzi, and P.~Frossard, ``{DeepFool}: A simple and
  accurate method to fool deep neural networks,'' in \emph{IEEE Conference on
  Computer Vision and Pattern Recognition}, 2016, pp. 2574--2582.

\bibitem[Fawzi et~al.(2017{\natexlab{b}})Fawzi, Moosavi-Dezfooli, Frossard, and
  Soatto]{fawzi2017classification}
A.~Fawzi, S.-M. Moosavi-Dezfooli, P.~Frossard, and S.~Soatto, ``Classification
  regions of deep neural networks,'' \emph{arXiv preprint arXiv:1705.09552},
  2017.

\bibitem[Poole et~al.(2016)Poole, Lahiri, Raghu, Sohl-Dickstein, and
  Ganguli]{ben2016expressivity}
B.~Poole, S.~Lahiri, M.~Raghu, J.~Sohl-Dickstein, and S.~Ganguli, ``Exponential
  expressivity in deep neural networks through transient chaos,'' in
  \emph{Advances in Neural Information Processing Systems}, 2016, pp.
  3360--3368.

\bibitem[Guo et~al.(2018)Guo, Rana, Cisse, and van~der
  Maaten]{guo2017inputtransform}
C.~Guo, M.~Rana, M.~Cisse, and L.~van~der Maaten, ``Countering adversarial
  images using input transformations,'' in \emph{International Conference on
  Learning Representations}, 2018.

\bibitem[Das et~al.(2017)Das, Shanbhogue, Chen, Hohman, Chen, Kounavis, and
  Chau]{das2017keeping}
N.~Das, M.~Shanbhogue, S.-T. Chen, F.~Hohman, L.~Chen, M.~E. Kounavis, and
  D.~H. Chau, ``Keeping the bad guys out: Protecting and vaccinating deep
  learning with {JPEG} compression,'' \emph{arXiv preprint arXiv:1705.02900},
  2017.

\bibitem[Bhagoji et~al.(2018)Bhagoji, Cullina, Sitawarin, and
  Mittal]{bhagoji2017dimensionality}
A.~N. Bhagoji, D.~Cullina, C.~Sitawarin, and P.~Mittal, ``Enhancing robustness
  of machine learning systems via data transformations,'' in \emph{52nd Annual
  Conference on Information Sciences and Systems (CISS)}, 2018, pp. 1--5.

\bibitem[Ilyas et~al.(2017)Ilyas, Jalal, Asteri, Daskalakis, and
  Dimakis]{ilyas2017robust}
A.~Ilyas, A.~Jalal, E.~Asteri, C.~Daskalakis, and A.~G. Dimakis, ``The robust
  manifold defense: Adversarial training using generative models,'' \emph{arXiv
  preprint arXiv:1712.09196}, 2017.

\bibitem[Samangouei et~al.(2018)Samangouei, Kabkab, and
  Chellappa]{samangouei2018defensegan}
P.~Samangouei, M.~Kabkab, and R.~Chellappa, ``Defense-{GAN}: {Protecting}
  classifiers against adversarial attacks using generative models,'' in
  \emph{International Conference on Learning Representations}, 2018.

\bibitem[Athalye et~al.(2018)Athalye, Carlini, and
  Wagner]{obfuscated-gradients}
A.~Athalye, N.~Carlini, and D.~Wagner, ``Obfuscated gradients give a false
  sense of security: Circumventing defenses to adversarial examples,'' in
  \emph{International Conference on Machine Learning}, 2018, pp. 274--283.

\bibitem[Raghunathan et~al.(2018{\natexlab{a}})Raghunathan, Steinhardt, and
  Liang]{raghunathan2018certified}
A.~Raghunathan, J.~Steinhardt, and P.~Liang, ``Certified defenses against
  adversarial examples,'' in \emph{International Conference on Learning
  Representations}, 2018.

\bibitem[Raghunathan et~al.(2018{\natexlab{b}})Raghunathan, Steinhardt, and
  Liang]{raghunathan2018semidefinite}
------, ``Semidefinite relaxations for certifying robustness to adversarial
  examples,'' in \emph{Advances in Neural Information Processing Systems},
  2018, pp. 10\,900--10\,910.

\bibitem[Wong and Kolter(2018)]{kolter2017provable}
E.~Wong and Z.~Kolter, ``Provable defenses against adversarial examples via the
  convex outer adversarial polytope,'' in \emph{International Conference on
  Machine Learning}, 2018, pp. 5283--5292.

\bibitem[Wong et~al.(2018)Wong, Schmidt, Metzen, and Kolter]{wong2018scaling}
E.~Wong, F.~Schmidt, J.~H. Metzen, and J.~Z. Kolter, ``Scaling provable
  adversarial defenses,'' in \emph{Advances in Neural Information Processing
  Systems}, 2018, pp. 8410--8419.

\bibitem[Sinha et~al.(2018)Sinha, Namkoong, and Duchi]{sinha2018certifiable}
A.~Sinha, H.~Namkoong, and J.~Duchi, ``Certifying some distributional
  robustness with principled adversarial training,'' in \emph{International
  Conference on Learning Representations}, 2018.

\bibitem[Mirman et~al.(2018)Mirman, Gehr, and Vechev]{mirman2018differentiable}
M.~Mirman, T.~Gehr, and M.~Vechev, ``Differentiable abstract interpretation for
  provably robust neural networks,'' in \emph{International Conference on
  Machine Learning}, 2018, pp. 3575--3583.

\bibitem[Hein and Andriushchenko(2017)]{hein2017formal}
M.~Hein and M.~Andriushchenko, ``Formal guarantees on the robustness of a
  classifier against adversarial manipulation,'' in \emph{Advances in Neural
  Information Processing Systems}, 2017, pp. 2266--2276.

\bibitem[Cisse et~al.(2017)Cisse, Bojanowski, Grave, Dauphin, and
  Usunier]{cisse2017parseval}
M.~Cisse, P.~Bojanowski, E.~Grave, Y.~Dauphin, and N.~Usunier, ``Parseval
  networks: Improving robustness to adversarial examples,'' in
  \emph{International Conference on Machine Learning}, 2017, pp. 854--863.

\bibitem[Marzi et~al.(2018)Marzi, Gopalakrishnan, Madhow, and
  Pedarsani]{isit_2018}
Z.~Marzi, S.~Gopalakrishnan, U.~Madhow, and R.~Pedarsani, ``Sparsity-based
  defense against adversarial attacks on linear classifiers,'' in \emph{IEEE
  International Symposium on Information Theory (ISIT)}, 2018, pp. 31--35.

\bibitem[Bafna et~al.(2018)Bafna, Murtagh, and Vyas]{bafna2018thwarting}
M.~Bafna, J.~Murtagh, and N.~Vyas, ``Thwarting adversarial examples: An $ {L}_0
  $-robust sparse fourier transform,'' in \emph{Advances in Neural Information
  Processing Systems}, 2018, pp. 10\,096--10\,106.

\bibitem[Fawzi et~al.(2018)Fawzi, Fawzi, and Fawzi]{fawzi2018adversarial}
A.~Fawzi, H.~Fawzi, and O.~Fawzi, ``Adversarial vulnerability for any
  classifier,'' in \emph{Advances in Neural Information Processing Systems},
  2018, pp. 1178--1187.

\bibitem[Katz et~al.(2017)Katz, Barrett, Dill, Julian, and
  Kochenderfer]{katz2017reluplex}
G.~Katz, C.~Barrett, D.~L. Dill, K.~Julian, and M.~J. Kochenderfer, ``Reluplex:
  {An} efficient {SMT} solver for verifying deep neural networks,'' in
  \emph{International Conference on Computer Aided Verification}.\hskip 1em
  plus 0.5em minus 0.4em\relax Springer, 2017, pp. 97--117.

\bibitem[Ehlers(2017)]{ehlers2017formal}
R.~Ehlers, ``Formal verification of piece-wise linear feed-forward neural
  networks,'' in \emph{International Symposium on Automated Technology for
  Verification and Analysis}.\hskip 1em plus 0.5em minus 0.4em\relax Springer,
  2017, pp. 269--286.

\bibitem[Cheng et~al.(2017)Cheng, N{\"u}hrenberg, and Ruess]{cheng2017maximum}
C.-H. Cheng, G.~N{\"u}hrenberg, and H.~Ruess, ``Maximum resilience of
  artificial neural networks,'' in \emph{International Symposium on Automated
  Technology for Verification and Analysis}.\hskip 1em plus 0.5em minus
  0.4em\relax Springer, 2017, pp. 251--268.

\bibitem[Dutta et~al.(2018)Dutta, Jha, Sankaranarayanan, and
  Tiwari]{dutta2018output}
S.~Dutta, S.~Jha, S.~Sankaranarayanan, and A.~Tiwari, ``Output range analysis
  for deep feedforward neural networks,'' in \emph{NASA Formal Methods
  Symposium}.\hskip 1em plus 0.5em minus 0.4em\relax Springer, 2018, pp.
  121--138.

\bibitem[Carlini and Wagner(2017)]{carlini2016distillationrefuted}
N.~Carlini and D.~Wagner, ``Towards evaluating the robustness of neural
  networks,'' in \emph{IEEE Symposium on Security and Privacy}, 2017, pp.
  39--57.

\bibitem[Daubechies(1992)]{daubechies1992ten}
I.~Daubechies, \emph{Ten Lectures on Wavelets}.\hskip 1em plus 0.5em minus
  0.4em\relax Siam, 1992, vol.~61.

\bibitem[Kurakin et~al.(2017)Kurakin, Goodfellow, and Bengio]{kurakin2016scale}
A.~Kurakin, I.~Goodfellow, and S.~Bengio, ``Adversarial machine learning at
  scale,'' in \emph{International Conference on Learning Representations},
  2017.

\bibitem[Dong et~al.(2018)Dong, Liao, Pang, Su, Zhu, Hu, and
  Li]{dong2017discovering}
Y.~Dong, F.~Liao, T.~Pang, H.~Su, J.~Zhu, X.~Hu, and J.~Li, ``Boosting
  adversarial attacks with momentum,'' in \emph{IEEE Conference on Computer
  Vision and Pattern Recognition}, 2018, pp. 9185--9193.

\bibitem[Kurakin et~al.(2018)Kurakin, Goodfellow, Bengio, Dong, Liao, Liang,
  Pang, Zhu, Hu, Xie, et~al.]{kurakin2018adversarial}
A.~Kurakin, I.~Goodfellow, S.~Bengio, Y.~Dong, F.~Liao, M.~Liang, T.~Pang,
  J.~Zhu, X.~Hu, C.~Xie \emph{et~al.}, ``Adversarial attacks and defences
  competition,'' in \emph{The NIPS'17 Competition: Building Intelligent
  Systems}.\hskip 1em plus 0.5em minus 0.4em\relax Springer, 2018, pp.
  195--231.

\bibitem[Nielsen(2015)]{deeplearning}
M.~A. Nielsen, \emph{Neural Networks and Deep Learning}.\hskip 1em plus 0.5em
  minus 0.4em\relax Determination Press, 2015.

\bibitem[He et~al.(2016)He, Zhang, Ren, and Sun]{he2016deep}
K.~He, X.~Zhang, S.~Ren, and J.~Sun, ``Deep residual learning for image
  recognition,'' in \emph{IEEE Conference on Computer Vision and Pattern
  Recognition}, 2016, pp. 770--778.

\end{thebibliography}
}

\begin{appendices}

\section{Proofs of Theorems} \label{appendix:theory}
\vspace{5pt}

\subsection{High SNR Regime of the Sparsifying Front End}
\vspace{3pt}

\begin{proposition} \label{prop:SNR_condition}
For sparsity level $K$ and perturbation $\be$ with $\Vert\be\Vert_\infty\leq\epsilon$, the sparsifying front end preserves the input coefficients if the following SNR condition holds:
\begin{equation*}
\mathrm{SNR} \, \triangleq \, \frac{\lambda}{\epsilon} \,> \gamma,
\end{equation*}
where $\lambda$ is the magnitude of the smallest non-zero entry of $\sparse(\Psi^T \bx)$ and $\gamma = 2\,\max_{k}{\left\Vert \bpsi_k\right\Vert_1}$.
\end{proposition}\vspace{1pt}

\begin{proof}
By Holder inequality, the SNR condition implies that
\begin{equation*}
\lambda > \, \epsilon\, \gamma \geq 2\, \max_{k}{ \left| \bpsi_k^T\be \right|} \geq \left| \bpsi_i^T\be \right| + \left| \bpsi_j^T\be\right| \;\; \forall\, i, j.
\end{equation*}
In particular, we can choose $i$ and $j$ such that
\begin{equation*}
\min_{i\in\support(\bx)} \left( \left|\bpsi^T_{i}\bx \right| - \left| \bpsi_i^T\be \right| \right) > \max_{j\notin\support(\bx)} { \left| \bpsi_j^T\be \right|}
\end{equation*}
where we have used the fact that $\lambda=\min_{i\in\support(\bx)} \left|\bpsi^T_{i}\bx\right|$. We can now use the triangle inequality to get
\begin{equation*}
\min_{k\in\support(\bx)}\left|\bpsi^T_{k} \left( \bx + \be \right)\right| > \max_{j\notin\support(\bx)}{ \left| \bpsi_j^T\be \right|}.
\end{equation*}
It is easy to see that this is equivalent to $\support(\bx+\be) = \support(\bx)$, which completes the proof.
\end{proof}\vspace{5pt}

\subsection{Linear Classifiers: Semi-White Box Scenario}
\vspace{3pt}

\begin{theorem} \label{theorem1}
As $K$ and $N$ approach infinity, $\Delta_{\mathrm{SW}}/K$ converges to $\epsilon\,\mu$ in probability,
i.e.\
\begin{equation*}
\lim_{K\to\infty} \prob \left( \left| \frac{\Delta_{\mathrm{SW}}}{K} - \epsilon\,\mu \right| \leq \delta \right) = 1 \quad \forall \; \delta >0.
\end{equation*}
Thus, the impact of adversarial perturbation in the case of semi-white box attack is attenuated by a factor of $K/N$ compared to having no defense.
\end{theorem}

\begin{proof}
Let us assume without loss of generality that $\support(\bx) = \{1,\dots,K\}$. We can rewrite the adversarial impact (\ref{eq:delta_sw}) as 
\begin{equation*}
\Delta_{\mathrm{SW}}= \epsilon \, \left|Z_K\right|, \; \textrm{where} \;\; Z_K = \sum_{i=1}^K \sign(\bw)^T \bpsi_k \bpsi_k^T \bw.
\end{equation*}
The following lemma provides an upper bound on the mean and variance of $Z_K$.
\begin{lemma} \label{var_covar}
The mean and variance of $Z_K$ are bounded by linear functions of $K$.
\begin{equation*}
\E[Z_K] = K \mu, \quad \vari (Z_K) \leq K \left(\sigma^2 + \mu^2\right). 
\end{equation*}
\end{lemma}
\begin{proof}
We can write $Z_K = \sum_{i=1}^K U_i V_i$, where
\begin{equation*}
U_i = \sum_{m=1}^N \psi_i \left[m\right] \, w_m,\;
V_i = \sum_{m=1}^N \psi_i \left[m\right] \, \sign(w_m).
\end{equation*}
We observe that for $i,j \in \{1,\dots,K\}$, $\E\left[U_i V_i\right] = \mu$, and
\begin{align*}
\vari\left(U_i V_i\right) &= \sigma^2 + \mu^2 -2 \mu^2 \sum_{m=1}^N \psi_i^4 \left[m\right],  \\
\cov\left(U_i V_i,U_j V_j\right) &= -2\mu^2 \sum_{m=1}^N \psi_i^2\left[m\right] \, \psi_j^2\left[m\right], \quad i \neq j.
\end{align*}
Hence we get $\E[Z_K] = K \mu$, and
\begin{align*}
\vari(Z_K) &= \sum_{i=1}^K \vari\left(U_i V_i\right) - \sum_{\substack{i,j=1 \\ i \neq j}}^K \cov\left(U_i V_i,U_j V_j\right) \\
&= K \left(\sigma^2 + \mu^2\right) -2\mu^2 \sum_{m=1}^N \, \sum_{i,j=1}^K \psi_i^2\left[m\right] \, \psi_j^2\left[m\right] \\
&\leq K \left(\sigma^2 + \mu^2\right).
\end{align*}
\end{proof}
\noindent We can now apply Chebyshev's inequality to $Y_K = {Z_K / K}$, using the bounds in the lemma to obtain
\begin{align}\label{eq:prob}
\prob \left(\left|Y_K-\mu\right| \leq \delta \right) \geq 1 - \frac{1}{K} \left(\frac{\sigma^2+\mu^2} {\delta^2}\right) \quad \forall \; \delta \geq 0.
\end{align}
Note that $\left|\Delta_{\mathrm{SW}}/K - \epsilon\,\mu\right| = \epsilon\,\left||Y_K| - \mu\right| \leq \epsilon\,\left|Y_K - \mu\right|$. The statement of the theorem follows by \eqref{eq:prob} and letting ${K\to\infty}$ in the above inequality.
\end{proof}

\subsection{Linear Classifiers: White Box Scenario}
\vspace{3pt}

\begin{lemma} \label{lemma:CLT}
$\bpsi_k^T \bw \rightarrow \mathcal{N}(0,\,\sigma^2)$ in distribution.
\end{lemma}
\begin{proof} \label{CLT_proof}
We show that we can apply Lindeberg's version of the central limit theorem, noting that $\bpsi_k^T\bw = \sum_{i=1}^N Y_i$, where $Y_i = \psi_k \left[i \right] \,w_i$ are independent random variables with $\E[Y_i]=0$ and $\vari (Y_i)=\sigma_i^2$, with $\sum_{i=1}^N \sigma_i^2=\sigma^2$.

Now, given $\delta>0$, we investigate the following quantity in order to check Lindeberg's condition:
\begin{equation*}
L(\delta,N) = \frac{1}{\sigma^2} \sum_{i=1}^N \E \left[ Y_i^2  \mathbbm{1}_{\left\{|Y_i|> \delta \sigma \right\}} \right].
\end{equation*}
From the $\ell_{\infty}$ assumptions on $\bpsi_k$ and $\bw$, we observe that
\begin{align*}
 \E \left[ \psi_k^2\left[i \right] w_i^2 \mathbbm{1}_{\left\{|Y_i|> \delta \sigma \right\}}\right] &\leq \littleo^2(1)\bigo^2(1) \Pr{\left(|Y_i|> \delta \sigma \right)} \\
  &= \littleo^2(1)\bigo^2(1) \Pr{\left(|w_i|> \frac{\delta \sigma}{\littleo(1)} \right)}.
 \end{align*}
Also note that $\forall\, \delta>0, \; \exists\, M$ s.t. $\forall N>M$, $|w_i |< \delta \sigma/\littleo(1)$ $\forall \,i \in \{1,\dots,N\}$. Hence we get $\lim_{N \to \infty} L(\delta,N) =0$, which is Lindeberg's condition. 
\end{proof}

\subsection{Neural Networks: High SNR Condition}
\vspace{3pt}

\begin{lemma}\ \label{lemma:clt}
$\bw^T_i \bx \to \mathcal{N}(0,\sigma_i^2)$ in distribution.
\end{lemma} 
\vspace{-5pt}
\begin{proof}
We show that we can apply Lindeberg's version of the CLT, noting that $\bw^T_i \bx = \sum_{j=1}^N U_j$ is the sum of independent random variables, where $U_j = x_j\, w_i\left[j\right]$ with $\E\left[U_j\right] = 0$, $\vari \left(U_j\right) = \sigma_i^2 x_j^2$ and $\sum_{j=1}^N \sigma_i^2 x_j^2 = \sigma_i^2$.

\noindent Now given a constant $c_1 = \Theta(1) >0$, we investigate the following quantity in order to check the Lindeberg condition:
\begin{equation*}
L(c_1,N) = \frac{1}{\sigma_i^2} \sum_{j=1}^N \E \left[ U_j^2  \mathbbm{1}_{\left\{|U_j|> c_1 \sigma_i \right\}} \right]
\end{equation*}

\noindent From the assumptions on $\bw_i$ and $\bx$, we observe that
\begin{align*}
 \E \left[ x_j^2 \, w_i^2\left[j \right] \,\mathbbm{1}_{\left\{ |U_j|> \delta \sigma_i \right\}} \right] 
 &
 \leq \littleo^2(1) \Theta(1) \Pr \left(|U_j|> c_1 \sigma_i \right) 
 \\ &
 = \littleo^2(1) \Theta(1) \Pr \left( \left|w_i\left[j \right]\right| > \frac{c_1 \sigma_i}{\littleo(1)} \right)
\end{align*}
The $\ell_\infty$ assumption on $\bw_i$ also implies that $\forall\, c_1>0, \; \exists\, N_0$ s.t. $|w_i[j] |< c_1 \sigma_i/\littleo(1), \,\;\forall \,j \in \{1,\dots,N\},\, N>N_0$. Hence we obtain that $\lim_{N \to \infty} L(\delta,N) =0$, which verifies the Lindeberg condition.
\end{proof}

\section{Additional Empirical Results} \label{appendix:exp} 
\vspace{5pt}

\subsection{SNR of sparsifying front end}

Figure \ref{fig_snr_svm} reports on the overlap in the $K$-dimensional supports of $\bx$ and $\bx+\be$ for attacks on linear SVM and CNN. We can observe that the SNR condition in Section \ref{sec:high_snr} is approximately satisfied for most of the images. 

Figure \ref{fig_svm_images} depicts perturbed images with various support overlaps for the SVM, showing an example in each scenario where the defense succeeds and another where the defense fails. For the fraction of images with low SNR, the adversary can cause significant image distortion by shifting the support of the $K$ selected basis functions, but however such distortions do not necessarily lead to misclassification. As we can observe from the histograms, the distribution of SNRs for images where the defense fails is almost identical to those for which it succeeds.

\begin{figure*}[]
\centering
\subfigure[Support overlap for white box attack on linear SVM with $\epsilon=0.1$ and $K=15$.]{\label{fig_snr_svm:a}\includegraphics[width=1.75\columnwidth]{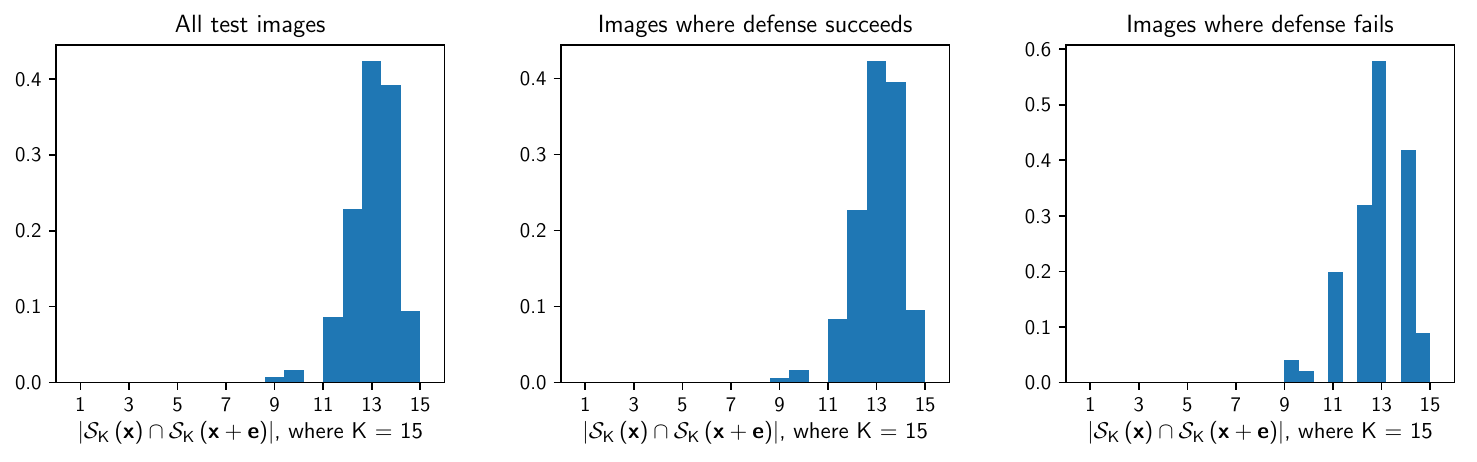}}
\par\smallskip
\subfigure[Support overlap for PGD attack (with 100 random restarts) on CNN with $\epsilon=0.2$ and $K=27$.]{\label{fig_snr_svm:b}\includegraphics[width=1.75\columnwidth]{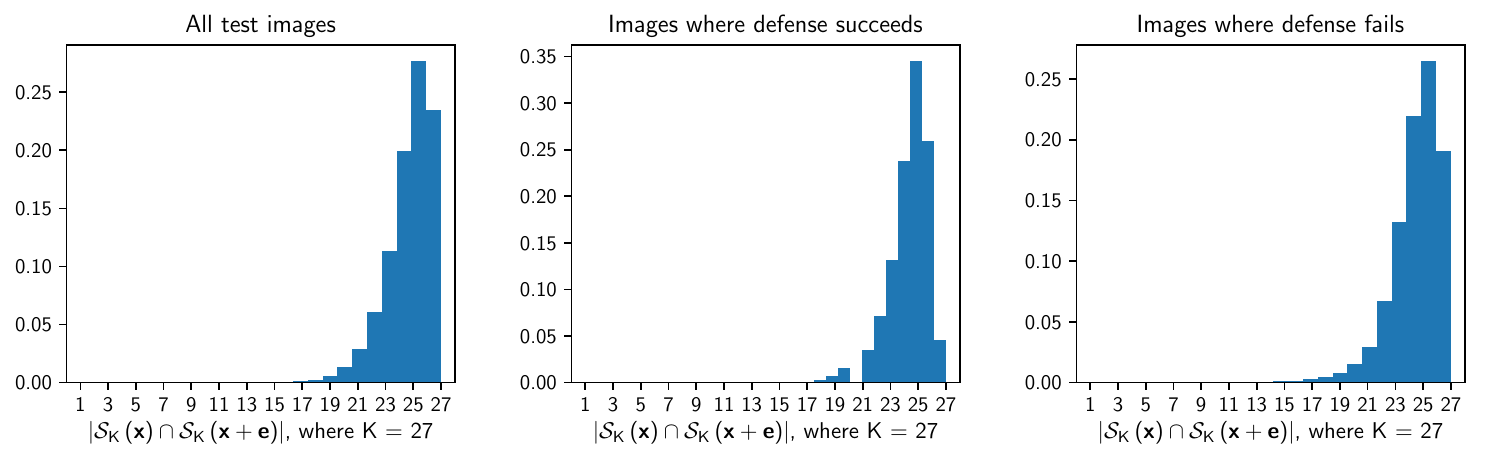}}
\caption{Histograms of support overlap, i.e.\ $\left| \support\left(\bx\right) \cap \support\left(\bx+\be\right)\right|$ for attacks on linear SVM and CNN. Plots are normalized to be probability densities, i.e.\ the area under each histogram sums to 1.}
\vspace{10pt}
\label{fig_snr_svm}
\end{figure*}

\subsection{SNR of ReLU units}

Figure \ref{fig_snr_relu} reports on the percentage of ReLU units that flip in one iteration of the iterative locally linear attack with $\delta=0.01$. When the defense is present, the high SNR condition in Section \ref{subsec:snr} is approximately satisfied. Figure \ref{fig_snr_niter_pgd} shows the evolution of the SNR condition with attack step for a 1000-step iterative FGSM attack. We can observe that on average, the percentage of ReLUs that flip in each iteration stays relatively small over attack iterations.

However, as the next section shows, it could be better for the adversary to try to make the most of the network's nonlinearity, for example by using iterative attacks with random initializations (PGD), so as to cause a large number of switches to flip to maximize the impact of the second term in Eq.\ \eqref{eq:dist}, Section \ref{subsec:snr}.

\begin{figure*}[]
\centering
\subfigure[No defense.]{\label{fig_snr_relu:a}\includegraphics[width=2\columnwidth]{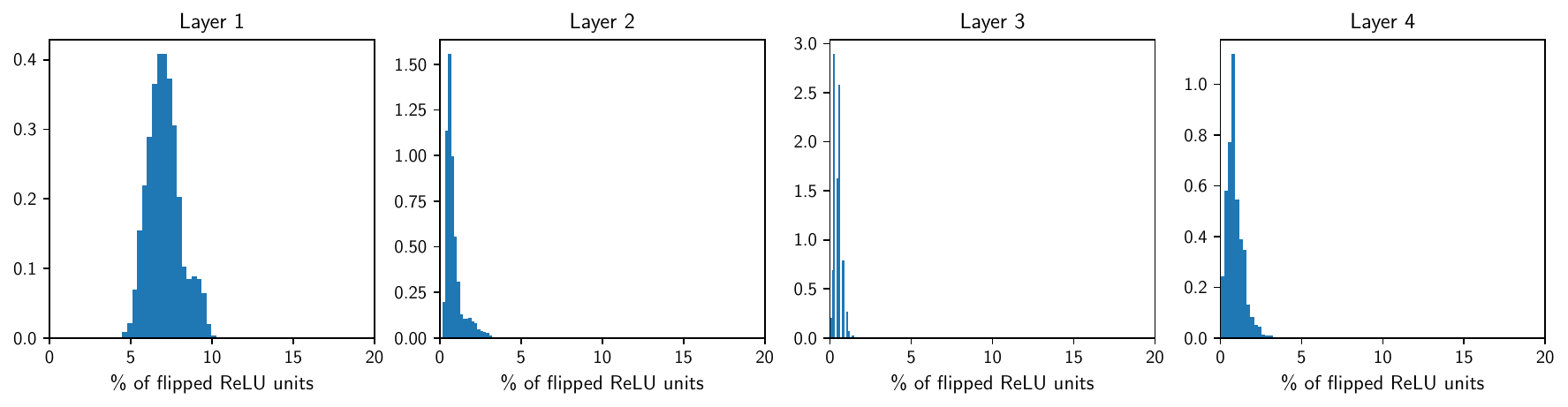}}
\subfigure[With sparsifying front end ($\rho = 3.5\%$).]{\label{fig_snr_relu:b}\includegraphics[width=2\columnwidth]{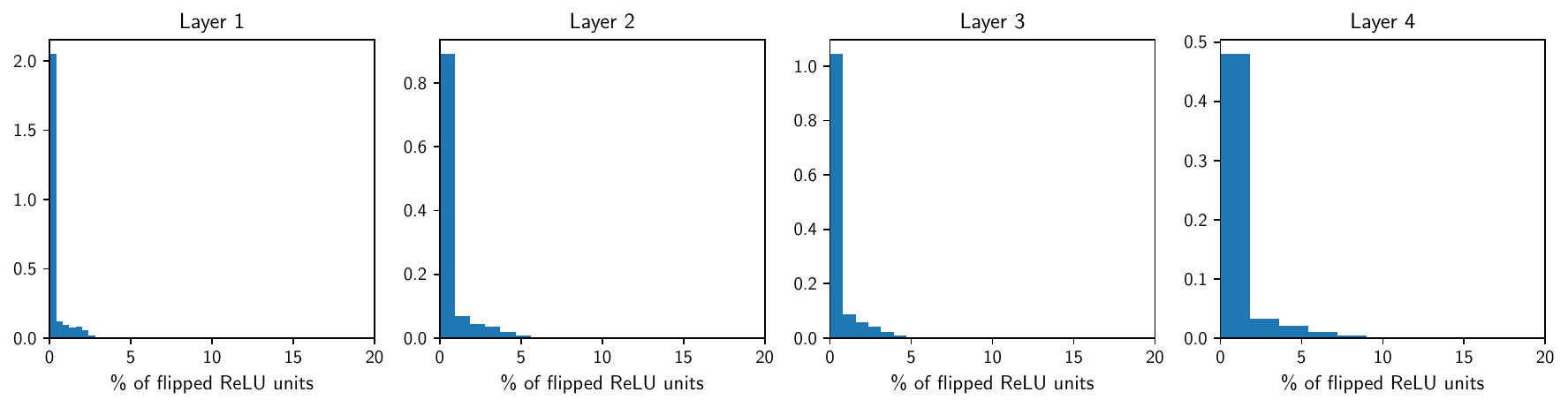}}
\caption{Histograms of the percentage of ReLU units that flip in a single step of the iterative locally linear attack with $\delta=0.01$, for the 4-layer CNN on MNIST. Plots are normalized to be probability densities.}
\label{fig_snr_relu}
\end{figure*}

\begin{figure*}[]
\centering
\subfigure{\label{fig_snr_niter_pgd:a}\includegraphics[width=0.99\columnwidth]{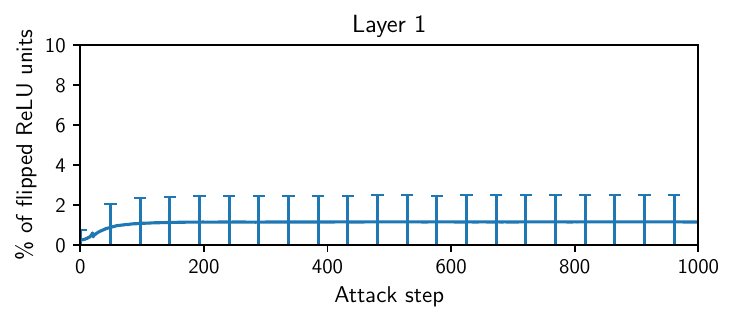}}\hfill
\subfigure{\label{fig_snr_niter_pgd:b}\includegraphics[width=0.99\columnwidth]{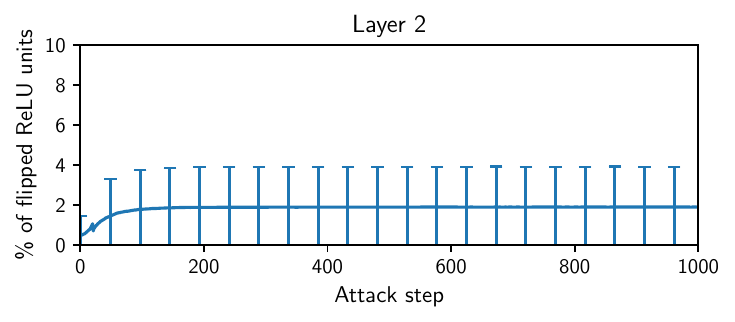}}
\par\smallskip
\subfigure{\label{fig_snr_niter_pgd:c}\includegraphics[width=0.99\columnwidth]{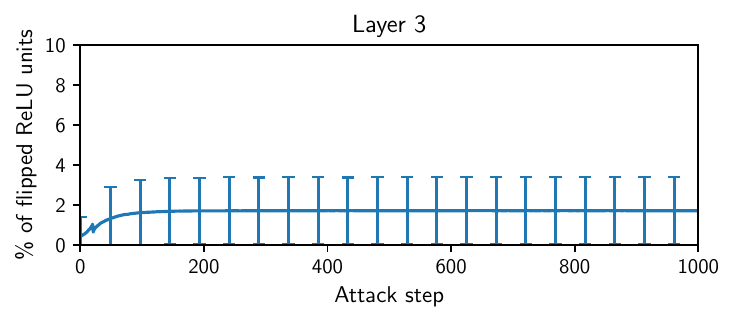}}\hfill
\subfigure{\label{fig_snr_niter_pgd:d}\includegraphics[width=0.99\columnwidth]{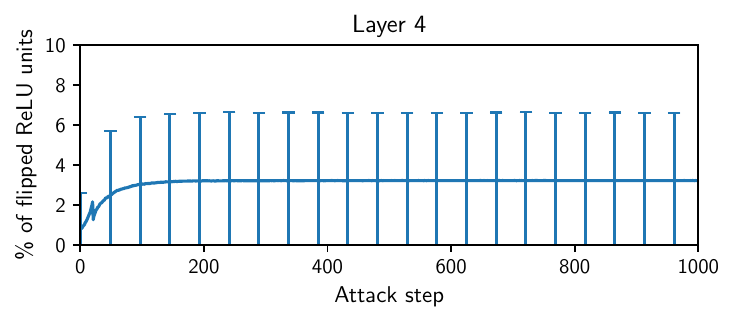}}
\caption{Plots showing the mean percentage of ReLU units that flip in each attack step, for a 1000-step iterative FGSM attack with $\delta=0.01$ and $\epsilon=0.2$ on the 4-layer CNN with defense ($\rho=3.5\%$). Error bars represent 1 standard deviation from the mean.}
\label{fig_snr_niter_pgd}
\end{figure*}

\begin{figure*}[]
\centering
\subfigure[Defense success, high support overlap.]{\label{fig_svm_images:a}\includegraphics[width=0.97\columnwidth]{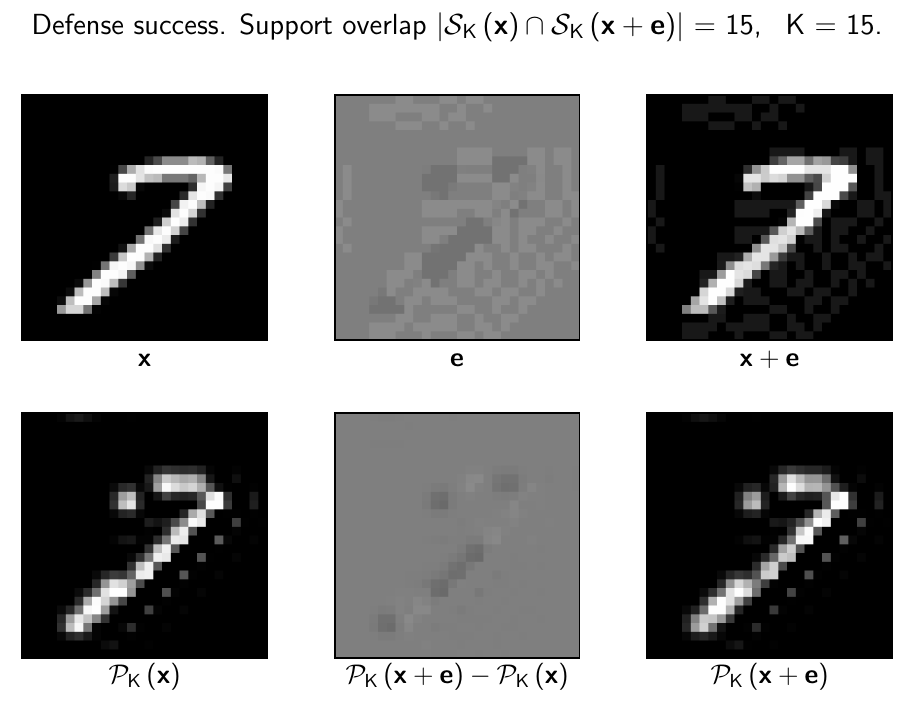}} \hfill
\subfigure[Defense success, low support overlap.]{\label{fig_svm_images:b}\includegraphics[width=0.97\columnwidth]{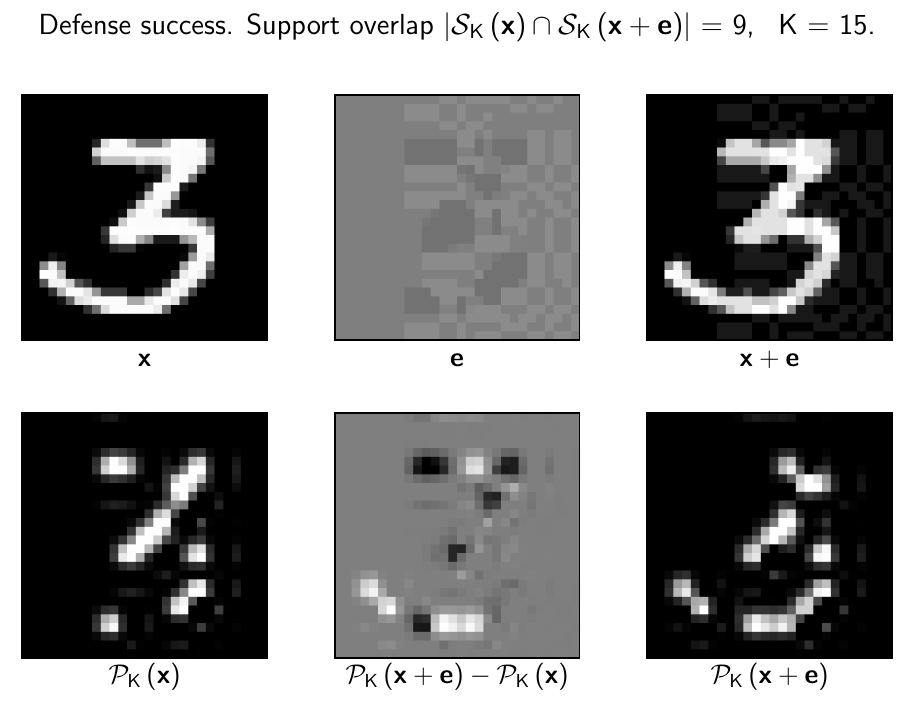}} 
\par\bigskip
\subfigure[Defense failure, high support overlap.]{\label{fig_svm_images:c}\includegraphics[width=0.97\columnwidth]{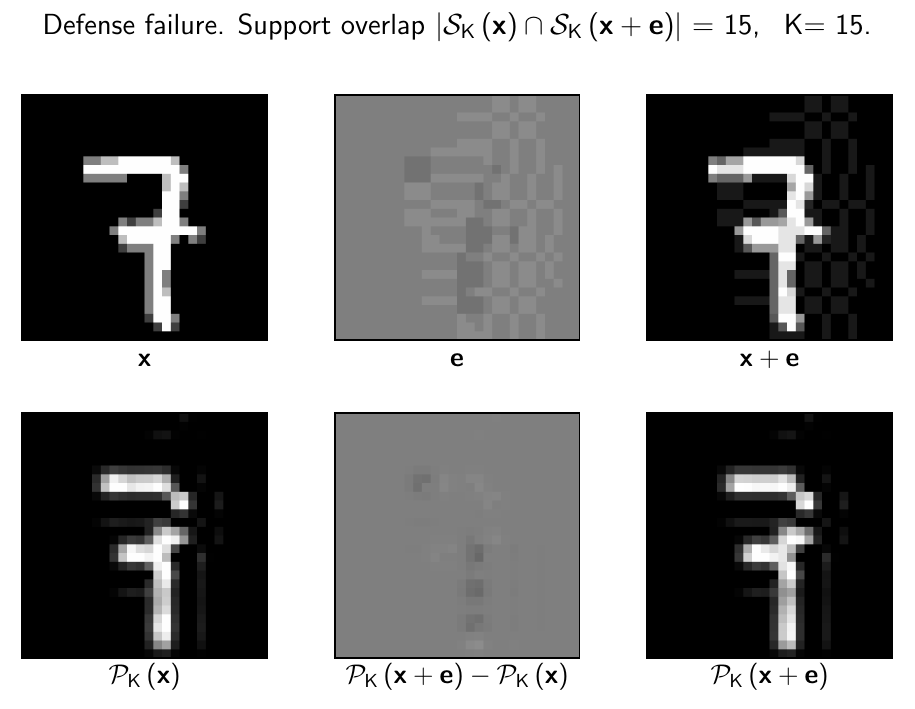}} \hfill
\subfigure[Defense failure, low support overlap.]{\label{fig_svm_images:d}\includegraphics[width=0.97\columnwidth]{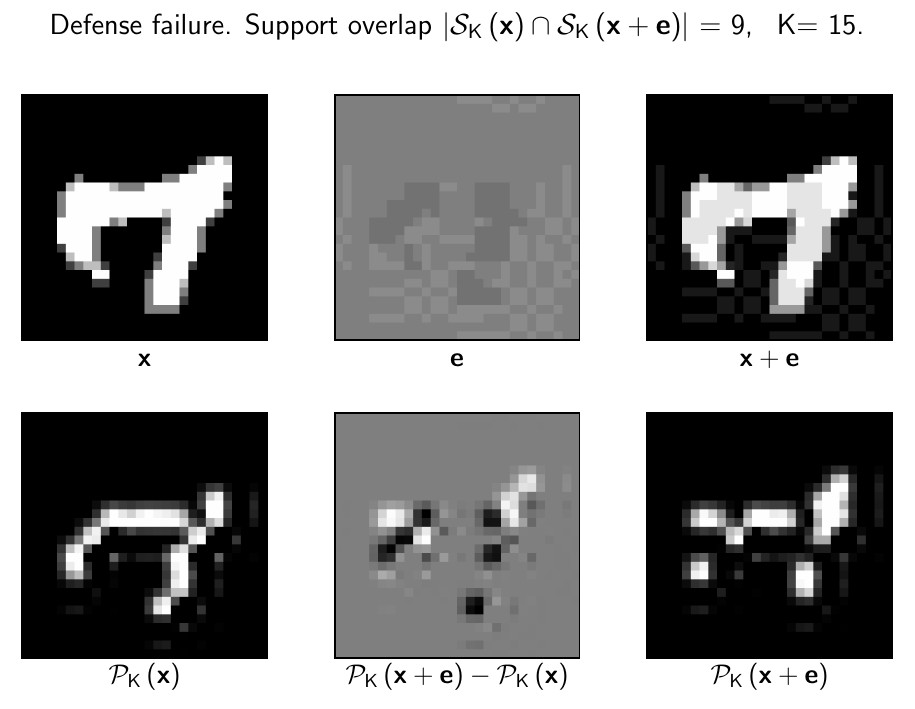}}
\caption{Sample images with low and high support overlap, for white box attack on the linear SVM with $\epsilon=0.1$ and $\rho=2\%$. The first row of each subfigure shows the original image, the perturbation and the attacked image, while the second row shows the effect of sparsification. Here $\proj\left(\bx+\be\right)$ denotes the projection of $\bx+\be$ onto its own $K$-dim.\ support i.e.\ $\support\left(\bx+\be\right)$.}
\label{fig_svm_images}
\end{figure*}

\subsection{More details on attack performance}

Table \ref{table_projections} reports on attack performance for the 3 versions described in Section \ref{exp_nn}: one that uses the backward pass differential approximation (BPDA) technique of \cite{obfuscated-gradients} to approximate the gradient of the front end as $\bm{1}$, a second version where the gradient is calculated as the projection onto the top $K$ basis vectors of the input, and a third version where we iteratively refine the projection. For iterative attacks, we refine the projection for 20 steps within each attack step.

 For PGD, we use 100 random restarts and report accuracy over the most successful restart(s) for each image.
We report on the effect of changing the per-iteration budget and number of steps for iterative FGSM in Table \ref{table_delta_niter}, with little change in accuracies beyond 100 iterations.

\begin{table}[!b]
\vspace{7pt}
\centering
\caption{Multiclass classification accuracies for 4-layer CNN on MNIST, with $\epsilon=0.2$ for attacks and $\rho = 3.5\%$ for defense. Iterative attacks were run for 1000 steps of $\delta=0.01$, except for the attacks marked $^\ast$ which use 100 steps of $\delta=0.05$. Projections were iterated for 20 steps within each attack step.}
\label{table_projections}
\begin{center}
\begin{small}
\begin{tabular}{lccc}
\toprule & 
\multicolumn{1}{c}{\begin{tabular}[l]{@{}c@{}}BPDA \\of $\bm{1}$\end{tabular}} & 
\multicolumn{1}{c}{\begin{tabular}[l]{@{}c@{}}Projections \end{tabular}} &
\multicolumn{1}{c}{\begin{tabular}[l]{@{}c@{}}Iterated \\projections \end{tabular}}\\
\midrule
Locally linear attack & 82.02 & 86.08 & 78.27\\
FGSM & 85.35 & 88.18 & 85.84 \\
\midrule
Iter.\ locally linear attack & 76.00 & 77.27 & \phantom{ }74.38$^\ast$
\\
Iter.\ FGSM & 74.97 & 79.88 & 75.33
 \\ 
Momentum iter.\ FGSM & 74.79 & 73.55 & ---
\\
PGD (100 restarts) & \phantom{ }64.62$^\ast$ & \phantom{ }65.48$^\ast$ & \phantom{ }61.04$^\ast$ 
\\
\bottomrule
\end{tabular}
\end{small}
\end{center}
\end{table}

\begin{table}[]
\vspace{7pt}
\centering
\caption{Iterative FGSM accuracies ($\epsilon=0.2$) as a function of the per-iteration budget $\delta$ and number of steps used, for the 4-layer CNN with front end ($\rho = 3.5\%$), using BPDA of $\bm{1}$.}
\label{table_delta_niter}
\begin{center}
\begin{small}
\begin{tabular}{cccc}
\toprule & 
\multicolumn{1}{c}{\begin{tabular}[l]{@{}c@{}}$\delta = 0.01$\end{tabular}} & 
\multicolumn{1}{c}{\begin{tabular}[l]{@{}c@{}}$\delta = 0.05$ \end{tabular}} &
\multicolumn{1}{c}{\begin{tabular}[l]{@{}c@{}}$\delta = 0.1$ \end{tabular}}\\
\midrule
20 steps & 80.84 & 75.52 & 75.32 \\[3pt]
100 steps & 75.13 & 75.17 & 75.02 \\[3pt]
1000 steps & 74.97 & 75.10 & 75.02 \\
\bottomrule
\end{tabular}
\end{small}
\end{center}
\end{table}

\subsection{Effect of sparsification on performance without attacks}

Sparsification causes a slight performance hit in the accuracy without attacks: for the 4-layer CNN, accuracy reduces from 99.31\% to 98.97\%. 
For the 2-layer NN,  we report implicitly on this effect in Figure \ref{fig_binary_nn} (the $\epsilon=0$ points): the accuracy goes from 99.33\% to 99.28\%.  
We note that sparsity has also been suggested purely as a means of improving classification performance (e.g., see \citet{makhzani2013ksparse}) and hence believe that with additional design effort, the performance penalty will be minimal. 

\subsection{Experiments on the Carlini-Wagner \texorpdfstring{$\ell_2$}{L2} attack}

Since the Carlini-Wagner $\ell_2$ attack does not have a fixed bound on the distance between adversarial and true images \citep{carlini2016distillationrefuted}, we report on histograms of distances in Figure \ref{fig_cw} (for the 4-layer CNN). We set the confidence level of the attack to 0, which corresponds to the smallest possible $\ell_2$ distance in the C\&W attack formulation. The final classification accuracy is 0.92\%, but 94.18\% of the perturbed images lie outside the $\ell_\infty$ budget of interest ($\epsilon$ = 0.2). The attack is successful in terms of causing misclassification, but since it is an $\ell_2$ attack, it fails to produce perturbations conforming to the $\ell_\infty$ budget. We generate the attack using the CleverHans library (v2.1.0), with default values of attack hyperparameters (listed below) and a backward pass differential approximation (BPDA) \citep{obfuscated-gradients} of $\bm{1}$ for the gradient of the front end.

\begin{lstlisting}[style=output]
 learning_rate = 5e-3, binary_search_steps = 5, max_iterations = 1000, initial_const = 1e-2.
\end{lstlisting}

\begin{figure*}[]
\centering
\subfigure[$\ell_\infty$ distance, with no defense.]{\label{cw_l_inf}\includegraphics[width=0.48\columnwidth]{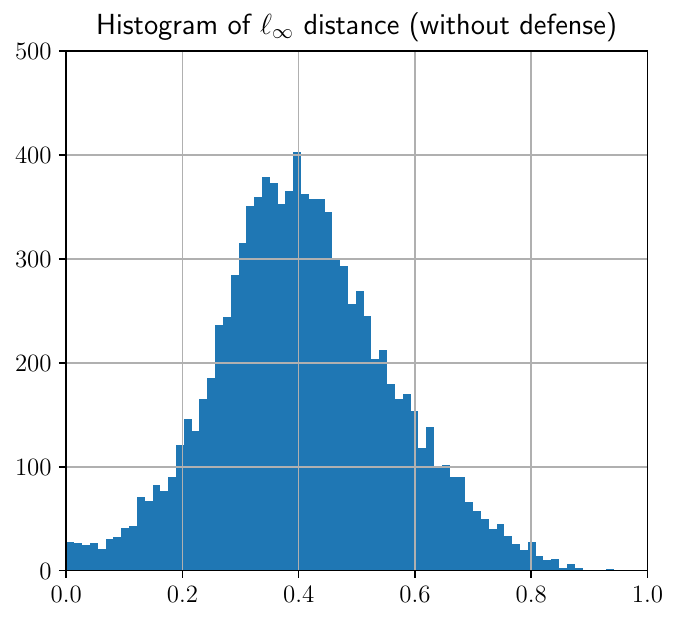}}
\hfill
\subfigure[$\ell_\infty$ distance, with sparsifying front end ($\rho$ = 3.5\%).]{\label{cw_l_inf_sp}\includegraphics[width=0.48\columnwidth]{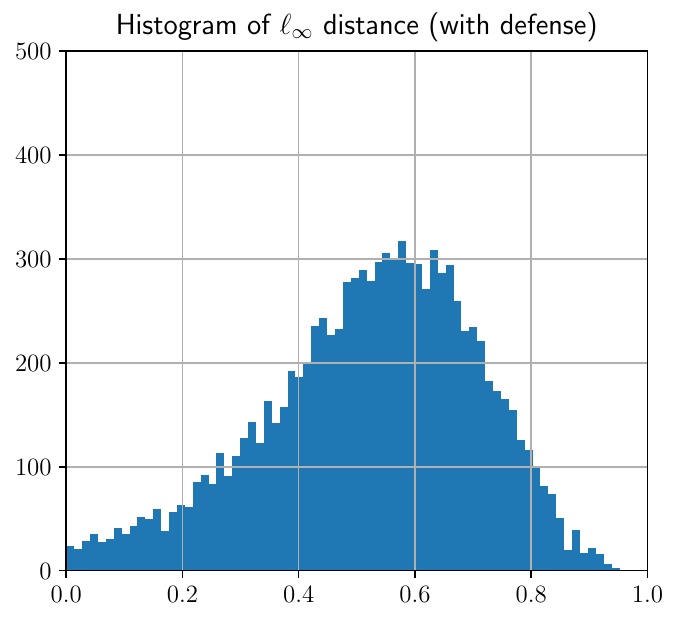}}
\hfill
\subfigure[$\ell_2$ distance, with no defense.]{\label{cw_l_2}\includegraphics[width=0.48\columnwidth]{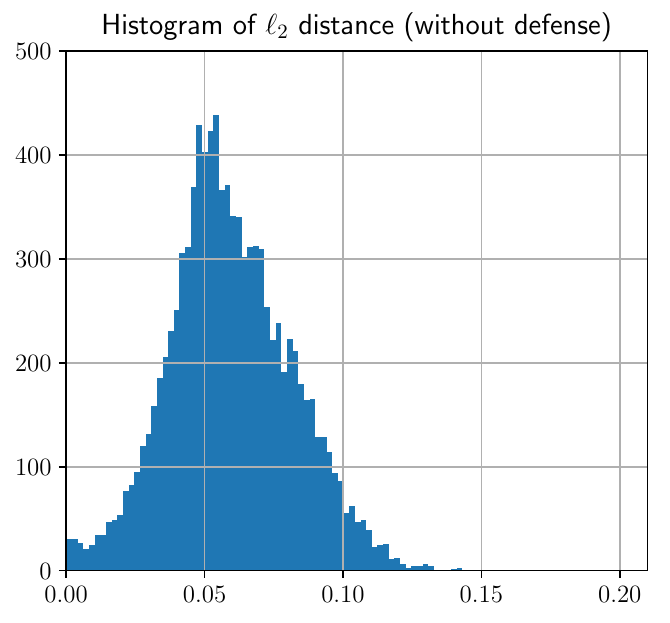}}
\hfill
\subfigure[$\ell_2$ distance, with sparsifying front end ($\rho$ = 3.5\%).]{\label{cw_l_2_sp}\includegraphics[width=0.48\columnwidth]{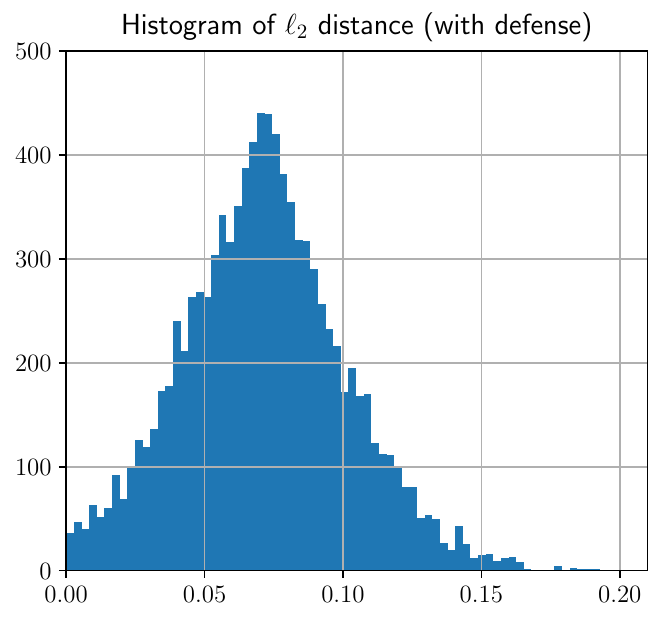}}
\vspace{-2pt}
\caption{Histograms of adversarial examples generated by the C\&W $\ell_2$ attack on the 4-layer CNN.}
\label{fig_cw}
\end{figure*}

\end{appendices}

\ifCLASSOPTIONcaptionsoff
  \newpage
\fi

\end{document}